%% file: main.tex
\def\year{2021}\relax
\documentclass[letterpaper]{article} 

\usepackage{times,latexsym}
\usepackage{amsfonts}
\usepackage{comment}
\usepackage{amsmath}
\usepackage{subcaption} 
\usepackage[T1]{fontenc}
\usepackage{enumitem}
\usepackage{booktabs}
\usepackage{makecell}
\usepackage{mathtools}

\usepackage{amsthm}
\usepackage{aaai21}  
\usepackage{times}  
\usepackage{helvet} 
\usepackage{courier}  
\usepackage[switch]{lineno} 
\usepackage[hyphens]{url}  
\usepackage{graphicx} 
\usepackage{natbib}  
\usepackage{caption} 
\newtheorem{theorem}{Theorem}

\newtheorem{lemma}[theorem]{Lemma}

\newtheorem{corollary}{Corollary}[theorem]

\title{Learning from Crowds by Modeling Common Confusions}
\author {
    Zhendong Chu, Jing Ma, Hongning Wang \\
}

\affiliations{
    Department of Computer Science, University of Virginia \\
    \{zc9uy, jm3mr, hw5x\}@virginia.edu 
}

\begin{document}

\maketitle

\begin{abstract} Crowdsourcing provides a practical way to obtain large amounts of labeled data at a low cost. However, the annotation quality of annotators varies considerably, which imposes new challenges in learning a high-quality model from the crowdsourced annotations. In this work, we provide a new perspective to decompose annotation noise into \emph{common noise} and \emph{individual noise} and differentiate the source of confusion based on instance difficulty and annotator expertise on a per-instance-annotator basis. We realize this new crowdsourcing model by an end-to-end learning solution with two types of noise adaptation layers: one is shared across annotators to capture their commonly shared confusions, and the other one is pertaining to each annotator to realize individual confusion. To recognize the source of noise in each annotation, we use an auxiliary network to choose from the two noise adaptation layers with respect to both instances and annotators. Extensive experiments on both synthesized and real-world benchmarks demonstrate the effectiveness of our proposed common noise adaptation solution. \end{abstract}

\input{introduction}

\input{related_work}
\input{model}

\input{method}

\input{experiment}

\section{Conclusion \& Future works} In this paper, we study the problem of learning from crowds with noisy annotations. Aside from the widely employed independent noise assumptions across annotators, we decompose annotation noise into common and individual confusions. We used neural networks to realize our probabilistic modeling of crowdsourced data, and estimate each component in our solution in an end-to-end fashion. Extensive empirical evaluations confirm the advantage of our solution in learning from complicated real-world crowdsourced data. Our solution is also flexible: it can be easily applied to any existing neural classifiers by simply connecting with the proposed noise adaptation layers.   In our current solution, all annotators share the same global confusion matrix. An interesting extension is to consider group-wise confusion, where we keep a shared confusion matrix for each annotator group, and identify the groups by optimization. It is also worthwhile to extend the solution to a proactive setting, e.g., probe annotators for more annotations so as to improve common confusion modeling.

\section*{Acknowledgments}
We thank our anonymous reviewers for their helpful comments. This work was supported by NSF 1718216, 1553568, and Department of Energy DE-EE0008227.

\input{broader_impact}
\bibliography{reference}
\input{appendix}
\end{document}

%% file: introduction.tex
\section{Introduction}
The availability of large amounts of labeled data is often a prerequisite for applying supervised learning solutions in practice. Crowdsourcing makes it possible to collect massive labeled data in both time- and cost-efficient manner \cite{buecheler2010crowdsourcing}. However, because of varying and unknown expertise of annotators, crowdsourced labels are usually noisy, which naturally lead to an important research problem: \emph{how to train an accurate learning model with only crowdsourced annotations?}

The first step to estimate an accurate learning model from crowdsourced annotations is to properly model the generation of such data. In this work, we focus on the crowdsourced classification problem. The seminal work from \citet{dawid1979maximum} (known as the DS model) assumes that each annotator has his/her own class-dependent confusion when providing annotations to instances. This is modeled by an annotator-specific confusion matrix, whose entries are the probability of flipping one class into another. The DS model has become the cornerstone of most learning from crowds solutions; and mainstream solutions perform label aggregation prior to classifier training: their key difference lies on different label aggregation methods based on the DS model \cite{venanzi2014community, zhang2014spectral, whitehill2009whose}. Recent developments focus more on unified solutions, where variants of the Expectation-Maximization (EM) algorithm are proposed to integrate label aggregation and classifier training \cite{albarqouni2016aggnet, cao2019max, raykar2010learning}. Typically, such solutions treat the classifier's predictions as latent variables, which are then mapped to the observed crowdsourced labels using individual confusion matrices of annotators. \citet{rodrigues2018deep} further fuse label inference and classifier training in an end-to-end approach using neural networks, where the gradient from label aggregation is directly propagated to estimate the annotators' confusion matrices. 
\citet{tanno2019learning} propose a similar solution but encourage the annotator confusion matrix to be close to an identity matrix by trace regularization.

All existing DS-model-based solutions assume noise in crowdsourced labels is only caused by individual annotators' expertise. However, it is not uncommon that different annotators would share common confusion about the labels. 
For example, when a \emph{bird} in an image is too small, every annotator has a chance to confuse it with an \emph{airplane} because of the background sky. We hypothesize that on an instance the annotator is confident about, he/she is more likely to use his/her expertise to provide a label (i.e., introducing individualized noise), while he/she would use common sense to label those unconfident ones. We empirically evaluate this hypothesis on two public crowdsourcing datasets, one for image labeling and one for music genre classification (more details of the datasets can be found in the Experiment Section), and visualize the results in Figure \ref{fig:noise_anlysis}. On both datasets, there are quite some commonly made mistakes across annotators. For example, on the image labeling dataset LabelMe, 61.0\% annotators mistakenly labeled \emph{street} as \emph{inside city} and 44.1\% of them mislabeled \emph{open country} as \emph{forest}; on the music classification dataset, 63.6\% annotators mislabeled \emph{metal} as \emph{rock} and 38.6\% of them mislabeled \emph{disco} as \emph{pop}. 
The existence of such shared confusions across annotators directly affects label aggregation: the majority of annotators are not necessarily correct, as their mistakes are no longer independent (e.g., those large off-diagonal entries in Figure \ref{fig:noise_anlysis}).
This is against the fundamental assumption in the DS model, and strongly urges new noise modeling to better handle real-world crowdsourced data.


\begin{figure}
\centering
\begin{subfigure}{.24\textwidth}
  \centering
  \includegraphics[height=2.6cm]{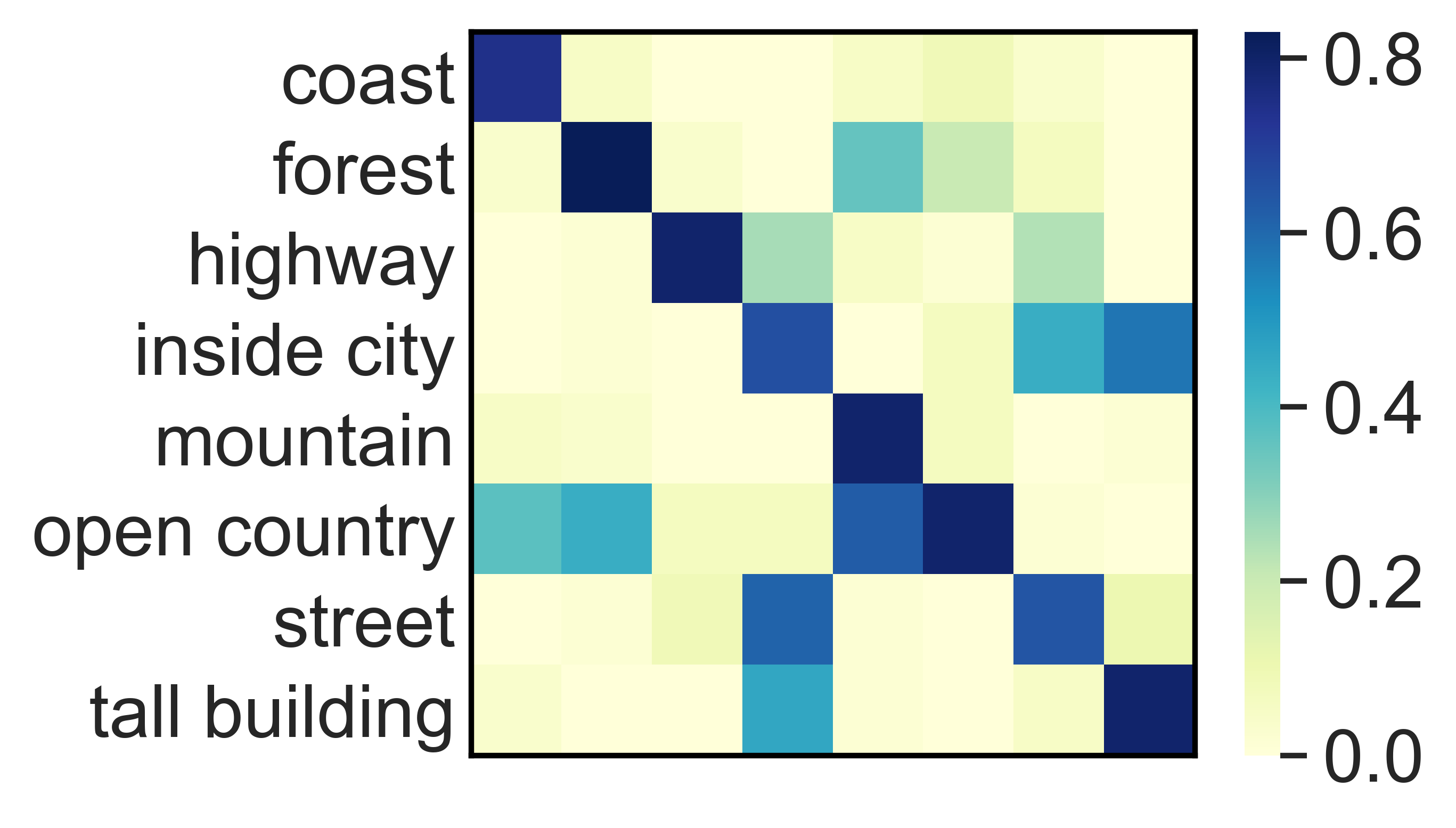}
  \caption{LabelMe}
  \label{fig:labelme}
\end{subfigure}%
\begin{subfigure}{.24\textwidth}
  \centering
  \includegraphics[height=2.44cm]{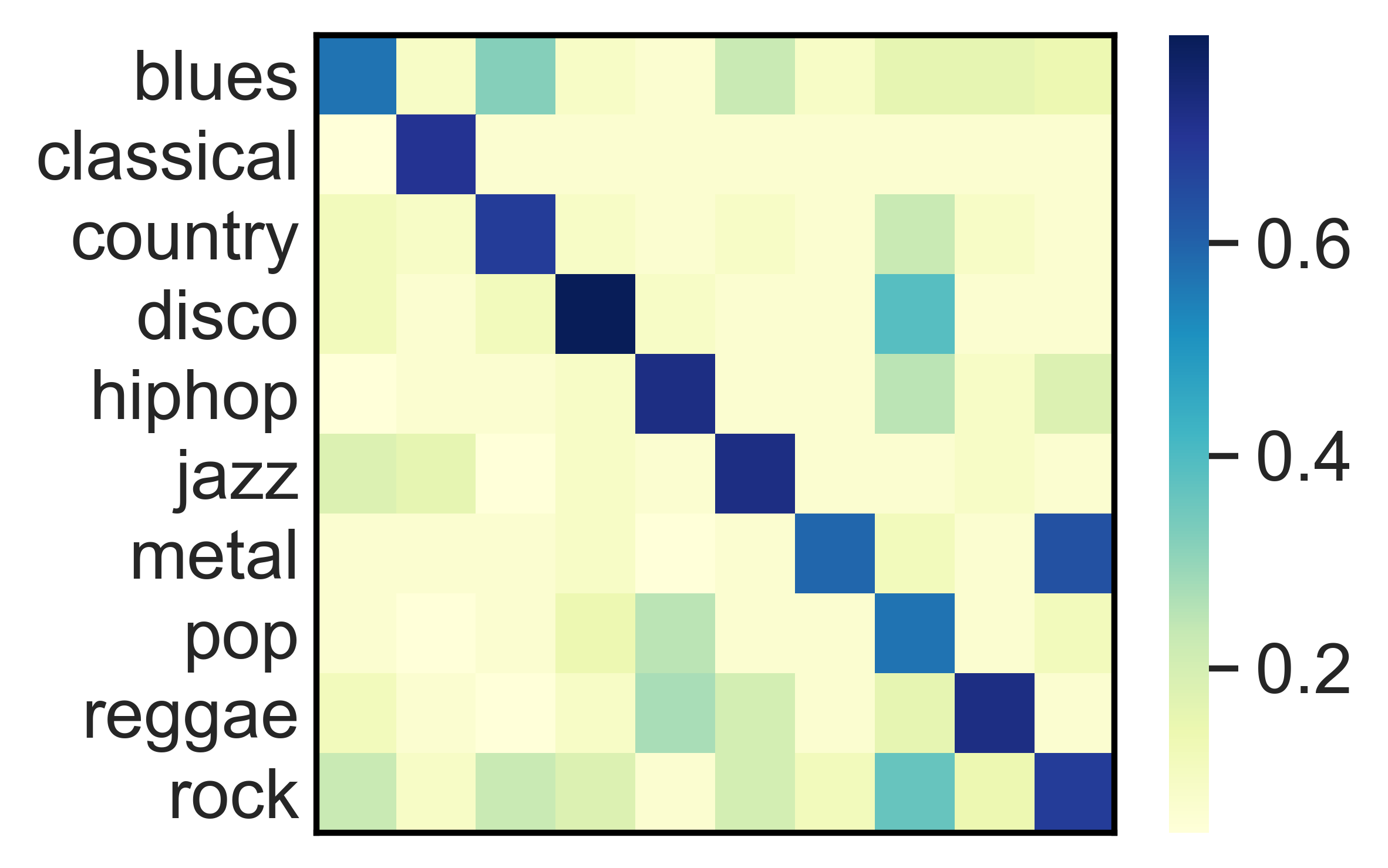}
  \caption{Music}
  \label{fig:music}
\end{subfigure}
\caption{Analysis of commonly made mistakes across annotators on two real-world crowdsourcing datasets. The value of each entry in the heatmap denotes the percentage of annotators with this confusion pair (e.g., mistakenly label \emph{street} as \emph{inside city} on LabelMe dataset).}
\label{fig:noise_anlysis}
\end{figure}


Moving beyond the independent noise assumption in the family of DS models \cite{dawid1979maximum, rodrigues2018deep}, we decompose annotation noise into two sources, \emph{common noise} and \emph{individual noise}, and differentiate the source of noise based on both annotators and instances. We refer to the annotation confusions shared across annotators as common noise, and model it by a global confusion matrix shared by all annotators. In the meanwhile, we also maintain annotator-specific confusion matrices for individual noise modeling.
We still treat ground-truth labels of instances as latent variables, but map them to noisy annotations by two parallel confusion matrices, to capture these different sources of noise. We determine the choice of confusion matrices on a per-instance-annotator basis, by explicitly modeling of annotator expertise and instance difficulty \cite{whitehill2009whose, yin2017aggregating}. 
To leverage the power of representation learning to model annotator expertise and instance difficulty, we realize all our model components using neural networks. In particular, we model the two types of confusion matrices as two parallel noise adaptation layers \cite{goldberger2016training}. For each annotator-instance pair, the classifier first maps the instance to a latent class label, then an auxiliary network decides which noise adaptation layer to map the latent class label to the observed annotation. 
Cross-entropy loss is counted on the predicted annotations for end-to-end training of these components. 
We name this approach \emph{CoNAL} - learning from crowds with \underline{\textbf{co}}mmon \underline{\textbf{n}}oise \underline{\textbf{a}}daptation \underline{\textbf{l}}ayers.  
Extensive experiments show considerable improvement of our new noise modeling approach against a rich set of baselines on two synthesized datasets, including a fully synthesized dataset and one based on CIFAR-10 dataset with various settings of noise generation, as well as two real-world datasets, e.g., LabelMe for image classification, and Music for music genre classification.

%% file: related_work.tex
\section{Related Works}
Several existing studies focused on modeling the different roles of instance and annotator in crowdsourced data. \citet{whitehill2009whose} model the accuracy of each annotation, which depends on instance difficulty and annotator expertise, to weigh each instance in final majority vote. \citet{welinder2010multidimensional} model each annotator as a multi-dimensional classifier and consider instance difficulty as single dimension latent variable. \citet{zhou2012learning} propose a minimax entropy principle on a probability distribution over annotators, instances and annotations, in which by minimizing entropy instance confusability and annotator expertise are naturally inferred. \citet{khetan2016achieving} and \citet{shah2016permutation} consider generalized DS models which model the instance difficulty. 
Instead of simply using a single scalar to model instance difficulty and annotator expertise as in previous works, we model them by learning their corresponding representations via an auxiliary network, which can better capture the shared statistical pattern across observed annotations.

Our method is closely related to several existing DS-based models considering relations among annotators; but it is also clearly distinct from them. \citet{kamar2015identifying} use a global confusion matrix to capture the identical mistakes by all annotators, and it is designed to replace the individual matrix when observations of an annotator are rare. Moreover, the choice of confusion matrix in this solution only depends on the number of annotations an annotator provided. This unnecessarily reflects the annotator expertise, as the task assignment is typically out of their control in crowdsourcing. \citet{venanzi2014community} and \citet{imamura2018analysis} cluster annotators to generate their own confusion matrices from a shared community-wide confusion matrix. However, the above approaches still assume a single underlying noise source, and thus they do not consider the difference between global (or community-level) and individual confusions. \citet{li2019exploiting} explore the correlation of annotation across annotators by classifying them into auxiliary subtypes under different ground-truth classes. However, the characteristics of each annotator are missing since they are only represented by a specific subtype. In our work, we still characterize individual annotators by modeling their own confusions.

%% file: model.tex
\section{Common Confusion Modeling in Crowdsourced Data}
In this section, we formulate our problem-solving framework for training classifiers directly from crowdsourced labels, based on the insight of common confusion modeling across annotators. We first describe the notations and our probabilistic modeling of the noisy annotation process, considering both common and individual confusions. This probabilistic model of noisy annotations is the basis of the end-to-end neural solution we develop in this paper. 

\subsection{Notations and Probabilistic Modeling}
Assume we have $N$ instances labeled by $R$ annotators out of $C$ possible classes. We define $\boldsymbol{x}_i$ as the feature vector of the $i$-th instance and $y_i^r$ as its label provided by the $r$-th annotator. Denote $z_i$ as the unobservable ground-truth label for the $i$-th instance, which is considered as a latent variable sampled from a multinomial distribution parameterized by $\{p(z_i=c|\boldsymbol{x}_i)\}_{c=1}^C$. For simplicity, we collectively define $X=\{\boldsymbol{x}_i\}_{i=1}^N$, $Y=\{y_i^r\}_{i=1, r=1}^{N, R}$ and $Z=\{z_i\}_{i=1}^N$. The final goal of learning from crowds is to obtain the classifier $P(Z|X)$ only with crowdsourced annotations $Y$.

Similar to the DS-based models (see Figure \ref{fig:ds_graph} for reference), the confusion of the $r$-th annotator is measured by an annotator-specific confusion matrix $\pi^r$, in which the $(z, z')$-element $\pi^r_{z, z'}$ denotes the probability that annotator $r$ will label the true label $z$ as $z'$. Aside from individual confusion, the key assumption of our solution is that annotation mistakes can also be introduced by common confusion, which is modeled by a globally shared confusion matrix $\pi^g$ across all annotators. We define the confusion matrices set as $\Pi = \{\pi^{1:R}, \pi^g\}$. We associate a Bernoulli random variable $s^r_i\sim B(\omega^r_i)$ with each annotation $y_i^r$ to differentiate the source of noise on it: $s^r_i$=1 if the confusion is caused by the common noise, where $w_i^r$ is the probability of the global confusion matrix being chosen by annotator $r$ on instance $i$ (see Figure \ref{fig:com_graph}). Denote the set of parameters governing the generation of $s_i^r$ across all annotations as $\Omega$.

Suggested by the successful practice in modeling crowdsourced data, we also impose the following two commonly made assumptions: 1) each annotator provides their annotations independently  \cite{dawid1979maximum}; and 2) each annotation is independent from the instance's features given the ground-truth labels \cite{yan2014learning, rodrigues2018deep}. We should note the first assumption is \emph{not} contradicting to our common confusion modeling: as the annotators can independently choose the shared common noise model to generate their annotations, the resulting observed annotations are no longer independent across annotators. As a result, the complete data likelihood of observed annotations under our model can be defined as, 
\begin{align}
    \label{eq:raw_likelihood}
    p\left(Y, Z| X, \Pi, \Omega\right) &\!=\! \prod_{i=1}^N\prod_{r=1}^R\sum^C_{z=1} p\left(y_i^r|z_i; \Pi,  \omega^r_i)p(z_i|\boldsymbol{x}_i\right), \\
     p\left(y_i^r|z_i; \Pi, \omega^r_i\right) &\!=\! \omega^r_ip\left(y_i^r|z_i, \pi^g\right) +  \left(1-\omega^r_i\right)p\left(y_i^r|z_i, \pi^r\right). \nonumber
\end{align}


Based on the above imposed problem structure, we derive an information-theoretical lower bound about  the resulting noise modeling quality. Let $\hat{Z}$ be the estimated true labels of all instances. Noise modeling quality is measured by the error rate given by $\mathcal{L}(\hat{Z}, Z) = \frac{1}{N}\sum_{i=1}^N\mathbb{I}\left(\hat{z}_i\neq z_i\right)$, where $\mathbb{I}(\cdot)$ is an indicator function. Given the ground-truth instance-specific class distribution $\rho_i = \{\rho_{ic}\}_{c=1}^C$ and confusion matrices $\Pi$, we have the following theorem about the lower bound of minimax error rate of our model.

\begin{theorem}
     The minimax error rate of our model is lower bounded by 
\begin{align}
\label{eq:lower_bound}
    \textup{inf}_{\hat{Z}}  \textup{sup}_{Z\in [C]^N} &\mathbb{E}\left[\mathcal{L}(\hat{Z}, Z)\right] \\
    &\geq  \frac{1}{N^2\textup{log}\,C}\sum_{i=1}^NF(\rho_i, \Pi, \Omega)
    - \frac{\textup{log}\,2}{N^2\textup{log}\,C}, \nonumber \\
    F(\rho_i, \Pi, \Omega) =& H(\rho_i) - \sum_{r=1}^{R}\sum_{c=1}^C\sum_{c'=1}^C\rho_{ic}\rho_{ic'}\Big(\omega^r_i\, \textup{KL}(\pi^g_{c*} \parallel \pi^g_{c'*})  \nonumber\\
    &+ (1-\omega_i^r)\,\textup{KL}\left(\pi^r_{c*}\parallel \pi^r_{c'*}\right)\Big). \nonumber
\end{align}
\label{trm1LB}
\end{theorem}
where 
$H(\rho_i)=-\sum_{c=1}^C\rho_{ic}\textup{log}\rho_{ic}$ is the entropy of ground-truth class distribution and $\pi_{c*}$ is the $c$-th row in confusion matrix $\pi$. The proof and further discussion of Theorem \ref{trm1LB} is provided in Appendix A. 

\noindent\textbf{Remarks.} This result extends the known lower bound result of DS models \cite{imamura2018analysis}. Lower bound on the error rate measures the difficulty of a crowdsourcing problem. Theorem \ref{trm1LB} suggests the proposed  decomposition has the potential to further reduce the lower bound, i.e., to obtain better inferred true labels. To understand this result, we should first note that the lower bound mainly depends on the KL distance between the class distributions conditioned on different ground-truth classes, as defined in $F(\rho_i, \Pi, \Omega)$, i.e., how two different classes will be confused with other classes. The more different they are (i.e., a larger KL distance), the easier one can differentiate the two from the observed noisy labels. For example, consider a crowdsourced dataset where an annotator labels a set of instances as \emph{airplane}; but among them, 50\% cases should be \emph{bird}, and the other 50\% should be \emph{spacecraft}. Intuitively, without any additional knowledge, it is hard to determine the true label when he/she labels an instance as \emph{airplane}. And this is asserted by Theorem \ref{trm1LB}: If we only used a single confusion matrix for this annotator, the conditional class distributions for \emph{bird} and \emph{spacecraft} will be pushed closer, because their entries on \emph{airplane} are close. This causes a smaller KL term in $F(\rho_i, \Pi, \Omega)$ between \emph{bird} and \emph{spacecraft} (e.g., setting $\omega_i^r$=0 for all instances in annotator $r$). But if we knew that the confusion between \emph{bird} and \emph{airplane} is caused by common noise, and the confusion between \emph{spacecraft} and \emph{airplane} is caused by individual noise, these mistakes could be attributed to two confusion matrices separately, which eliminates the misleading similarity between the conditional probabilities for \emph{bird} and \emph{spacecraft} caused by \emph{airplane}. 

\begin{figure}
\centering
\begin{subfigure}{.2\textwidth}
  \centering
  \includegraphics[width=4.2cm]{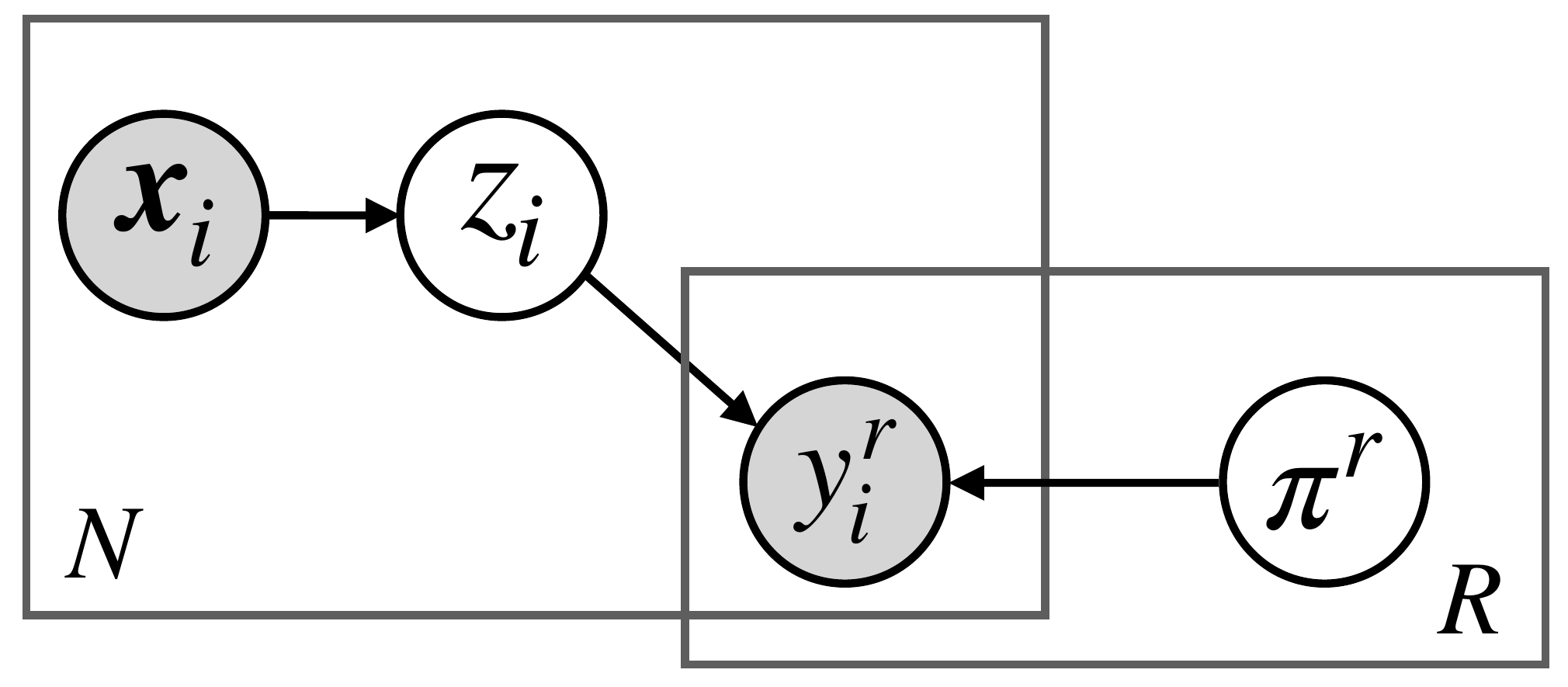}
  \caption{DS model.}
  \label{fig:ds_graph}
\end{subfigure}%
\begin{subfigure}{.32\textwidth}
  \centering
  \includegraphics[width=4.2cm]{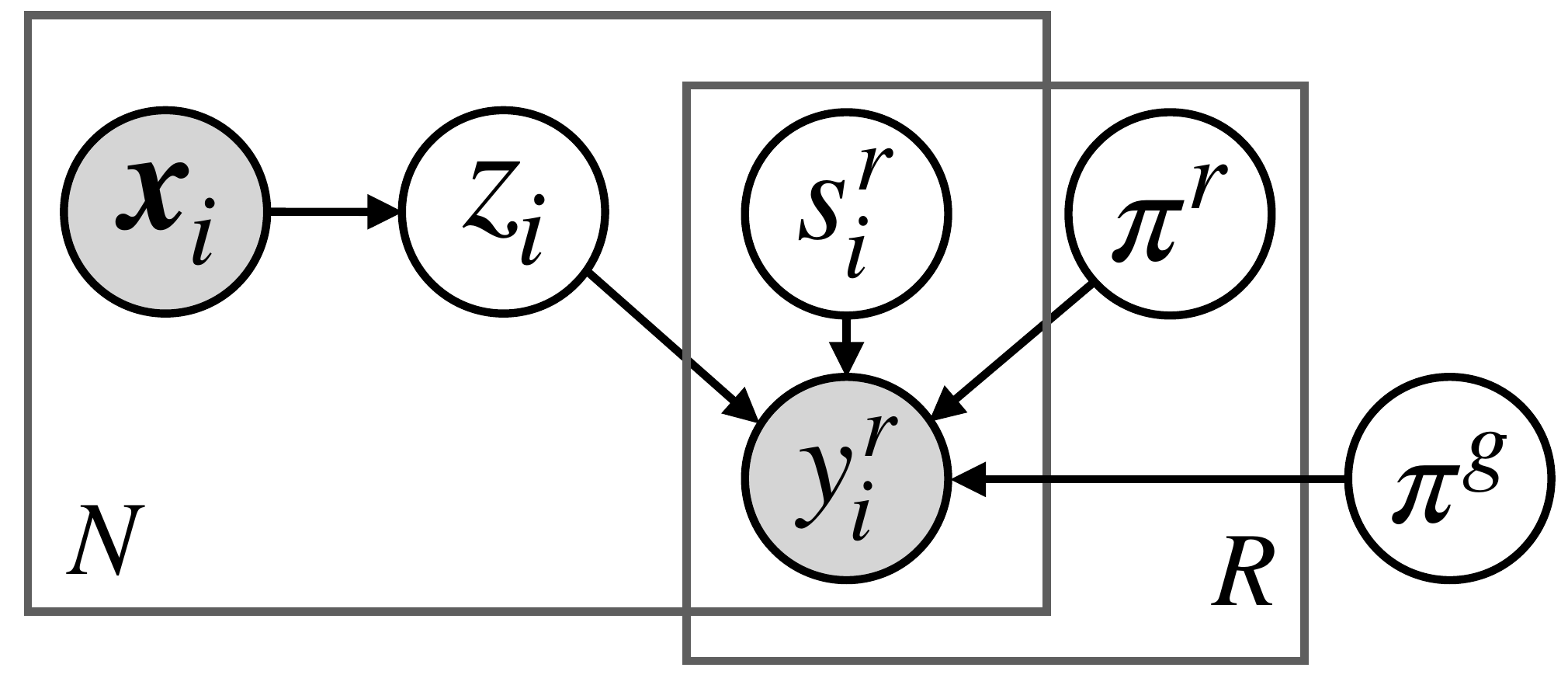}
  \caption{Common noise model.}
  \label{fig:com_graph}
\end{subfigure}
\caption{Graphical model presentations of DS model and our common noise model.}
\label{fig:graphical_model}
\end{figure}


%% file: method.tex
\subsection{End-to-end Learning Framework}

To apply our noise modeling in crowdsourced data, we need to estimate the confusion matrices $\Pi$ together with the classifier. Instead of building a vanilla tabular model for them, we realize them using neural models, to take advantage of the power of representation learning.
In particular, we map the output of the classifier to noisy annotations by two types of confusion layers, which we refer to as noise adaptation layers \cite{goldberger2016training}. We also introduce an auxiliary network that takes both annotator and instance as input to predict the choice of these two noise adaptation layers. 
Since we treat the ground-truth label of an instance as a latent variable, the Expectation Maximization (EM) algorithm becomes a natural choice for model learning, as typically done in literature \cite{albarqouni2016aggnet, rodrigues2018deep, bertsekas2014constrained}. For the integrity of work, we provide the derived EM algorithm in Appendix B for interested readers. However, the EM-based algorithm has several clear drawbacks in our solution: 1) In crowdsourced data, because the annotators typically only label a small proportion of instances, EM-based algorithm becomes very sensitive to the initialization of model parameters. It can easily cause instability issues in training a neural network model. 2) In every EM iteration, we need to retrain the neural network, which causes a huge overhead when handling large networks. Instead, we take an end-to-end approach to jointly perform latent variable inference and model parameter estimation. We define cross-entropy loss on the observed annotations and use error back-propagation to update the classifier's output and the network parameters simultaneously. 

\begin{figure}[!htp]
    \centering
    \includegraphics[width=8.2cm]{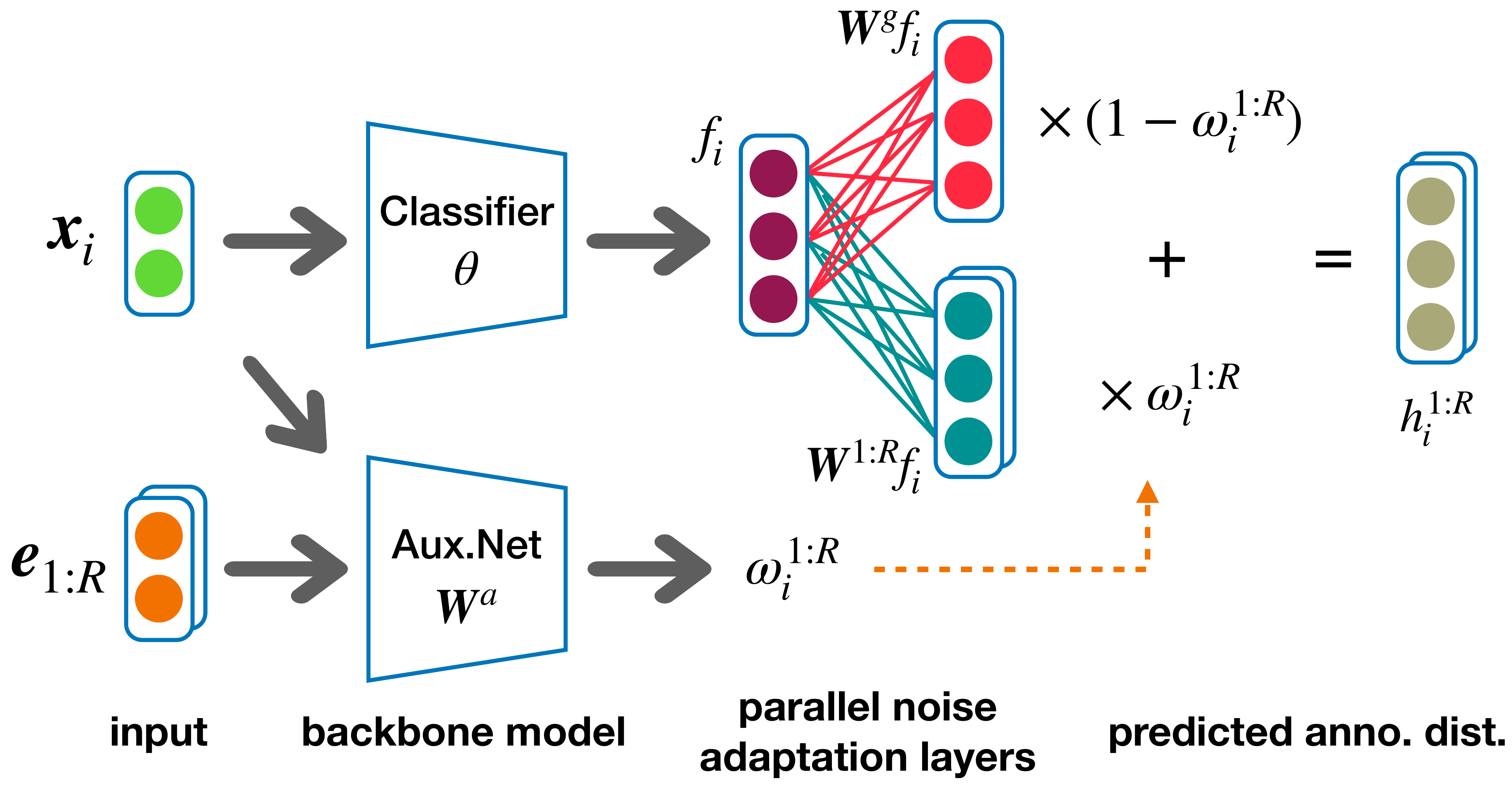}
    \caption{Overview of our framework for classification with 3 classes and $R$ annotators. }
    \label{fig:framework}
\end{figure}

We construct a neural network classifier with non-linear intermediate layers and a softmax output layer. The probability distribution of the predicted true label $z_i$ given the instance feature vector $\boldsymbol{x}_i$ is thus specified as $p_{\boldsymbol{\theta}}(z_i|\boldsymbol{x}_i)$, where $\boldsymbol{\theta}$ is the network parameter set including the softmax layer. We denote the immediate output of the classifier as $f_i=f(\boldsymbol{x}_i)\in \mathbb{R}^C$. We then use noise adaptation layers to map the classifier's output into noisy annotations, which are implemented by introducing additional softmax output layers on top of the output layer of the classifier (see overview in Figure \ref{fig:framework}). The weight matrices of the noise adaptation layers resemble confusion matrices $\Pi$ in a probabilistic sense. The output of the noise adaptation layer is thus the probability distribution of predicted annotation $p_{\boldsymbol{W}}(\hat{y}_i^r|f(\boldsymbol{x}_i))$, where $\boldsymbol{W}$ is the parameter set of the noise adaptation layer.


We consider two types of noise adaptation layers: one individual noise adaptation layer for every annotator parameterized by $\boldsymbol{W}^r$, and a common noise adaptation layer shared across all annotators parameterized by $\boldsymbol{W}^g$. The final probability distribution of annotations is obtained as,
\begin{equation*}
    p(\hat{y}_i^r|\boldsymbol{x}_i) = \omega_i^r\, p_{\boldsymbol{W}^g}(\hat{y}_i^r|f(\boldsymbol{x}_i)) + (1-\omega_i^r)\,p_{\boldsymbol{W}^r}(\hat{y}_i^r|f(\boldsymbol{x}_i)).
\end{equation*}
where $\omega_i^r$ governs the distribution that the mistake of annotator $r$ on instance $i$ is caused by common confusion $\pi^g$, denoted by the noise source indicator $s^r_i$.

As $s^r_i$ is unobservable, we introduce an auxiliary network to model $s^r_i\sim B(\omega_i^r)$ by parameterizing it over annotator expertise and instance difficulty, both of which are modeled via learnt representations by the auxiliary network. 
Specifically, as in our problem setup, every instance is associated with raw features, the auxiliary network takes instance feature $\boldsymbol{x}_i$ as input for learning instance $i$'s embedding $\boldsymbol{v}_i$. The same can be applied to annotator $r$, if any raw feature  $\boldsymbol{e}^r$ is available about the annotator, otherwise we use its one-hot encoding as input for learning annotator embedding $\boldsymbol{u}_r$. Then $\omega_i^r$ can be obtained as follows,
\begin{align}
    \begin{gathered}
        \boldsymbol{v}_i = \boldsymbol{W}_v\boldsymbol{x}_i + b_v, \boldsymbol{u}_r = \boldsymbol{W}_u\boldsymbol{e}_r + b_u, \\   \omega_{i}^r = \sigma(\boldsymbol{u}_r^\top  \boldsymbol{v}_i).
    \end{gathered}
    \label{eq:proportion}
\end{align}
where $(\boldsymbol{W}_v, b_v)$ and $(\boldsymbol{W}_u, b_u)$ are weight matrices and bias terms for annotator and instance embeddings, and $\sigma$ is a sigmoid function. To simplify our notations, we collectively refer the parameters in this auxiliary network as $\boldsymbol{W}^a$.
To avoid the magnitude of learnt $\boldsymbol{u}$ and $\boldsymbol{v}$ becoming extremely large or small, which causes numerical issues in estimating $\omega_{i}^r$, we normalize the learnt annotator and instance embeddings before computing their inner product.

Based on the above full specifications of our probabilistic modeling using neural networks, we are ready to estimate the network parameters. 
We can easily verify that, maximizing the likelihood of observed annotations given the input feature vectors as defined in Eq \eqref{eq:raw_likelihood} is equivalent to minimizing the cross-entropy loss between the observed annotations and predicted annotation distributions,
\begin{equation*}
    \mathcal{L}(\boldsymbol{\theta},  \boldsymbol{W}^g, \boldsymbol{W}^{1:R}, \boldsymbol{W}^a) = -\frac{1}{N}\sum_{i=1}^N\sum_{r=1}^R\sum_{j=1}^Cy^r_{ij}\textup{log}\,p_j(\hat{y}^r_{i}|\boldsymbol{x}_i).
\end{equation*}
where $y^r_{ij}=1$ if $y^r_i=j$; otherwise $y^r_{ij}=0$; and $p_j(\hat{y}^r_{i}|\boldsymbol{x}_i)$ refers to the $j$-th entry of the predicted annotation distribution. All parameters can be trained by back-propagation using gradient descent techniques, such as Adam \cite{kingma2014adam} and SGD \cite{goodfellow2016deep}. Once trained, in the testing phase, we can directly use the classifier to make predictions on new instances.

The gradient flow in back-propagation reveals how our common confusion modeling handles crowdsourced data. In the context of classification, we can simply view the introduced noise adaptation layer as performing a projection of gradients; and with a slight abuse of notations, we denote the output of our noise adaptation layers as $h^r_i=\omega_i^r\boldsymbol{W}^gf_i + (1-\omega_i^r)\boldsymbol{W}^rf_i$. Under the chain rule, the gradients are naturally decoupled with respect to different sources of noise,
\begin{equation}
    \frac{\partial \mathcal{L}}{\partial f_i} = \sum_{r=1}^R\frac{\partial \mathcal{L}}{\partial h^r_i}\frac{\partial h^r_i}{\partial f_i} = \sum_{r=1}^R \omega^r_i\frac{\partial \mathcal{L}}{\partial h^r_i}\boldsymbol{W}^g + (1-\omega^r_i) \frac{\partial \mathcal{L}}{\partial h^r_i} \boldsymbol{W}^r
    \label{eq:grad_flow}.
\end{equation}
It clearly shows confusion matrices reshape the gradients, which informs the classifier layer what the true label should be on an instance given its noisy annotations. The importance of each confusion matrix in shaping the classifier is determined by $\omega^r_i$, which infers the source of noise based on annotator expertise and instance difficulty. 

\begin{figure*}[!htp]
\begin{subfigure}{1\textwidth}
  \centering
  \includegraphics[width=15.1cm]{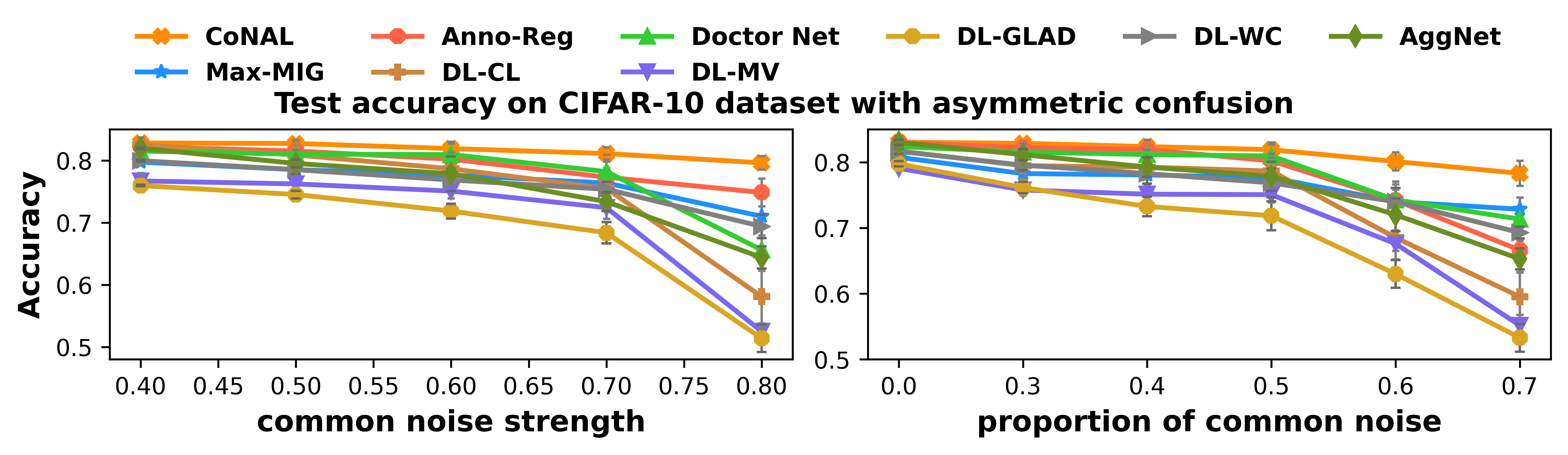}
  \label{fig:sub3}
\end{subfigure}

\begin{subfigure}{1\textwidth}
  \centering
  \includegraphics[width=15.1cm]{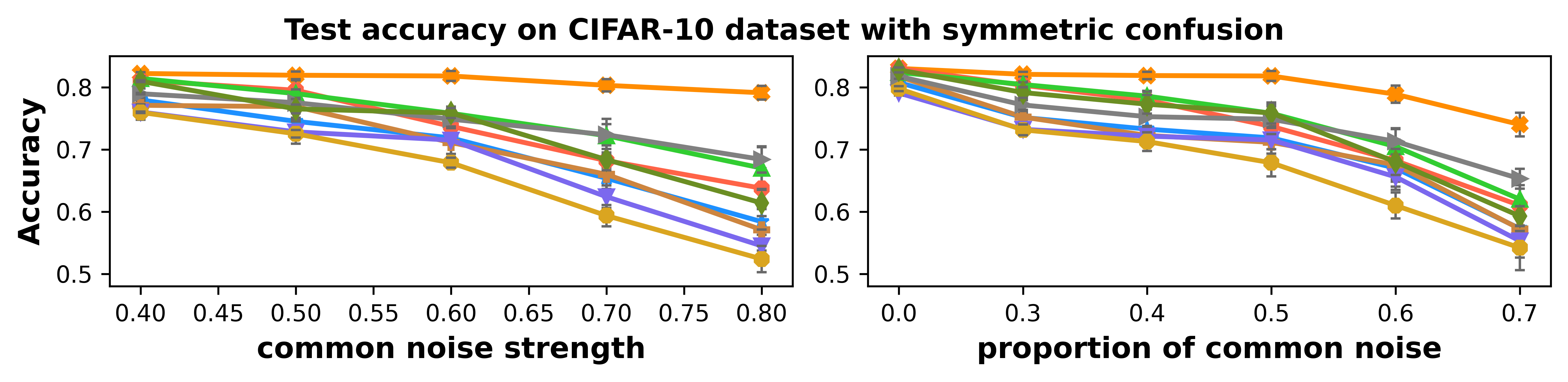}
  \label{fig:sub3}
\end{subfigure}
\caption{Results on CIFAR-10 dataset.}
\label{fig:cifar_results}
\end{figure*}

The gradients in Eq \eqref{eq:grad_flow} also suggest a potential bottleneck of our proposed solution: if the common and individual noise adaptation layers are unidentifiable, we cannot correctly attribute the noise, which is the key for our solution to perform according to Theorem \ref{trm1LB}. To avoid this, we add $\ell_2$-norm on the difference between the common and individual noise adaptation layers as a regularization term, to enforce them to be different. This presents our final loss function,
\begin{align*}
    \mathcal{L}(\boldsymbol{\theta},  \boldsymbol{W}^g, \boldsymbol{W}^{1:R}, \boldsymbol{W}^a) = &-\frac{1}{N}\sum_{i=1}^N\sum_{r=1}^R\sum_{j=1}^Cy^r_{ij}\textup{log}\,p_j(\hat{y}^r_i|\boldsymbol{x}_i)  \\
    &- \lambda\sum_{r=1}^R\left \| \boldsymbol{W}^g - \boldsymbol{W}^r \right \|_2
\end{align*}
where $\lambda$ is a hyper-parameter to control regularization.


%% file: experiment.tex
\section{Experiments}
\label{sec:experiment}
We evaluate our method on both synthesized and real-world datasets. We consider a rich set of related solutions as our baselines, which can be divided into two categories:

\noindent 1) Methods with simple noise models. \textbf{DL-MV}: it learns a neural network classifier with labels aggregated by majority voting. \textbf{DL-CL} \cite{rodrigues2018deep}: it learns a neural classifier with designated layers to fit individual annotator confusions (so-called crowd layer). \textbf{Anno-Reg} \cite{tanno2019learning}: it improves DL-CL by imposing additional trace regularization on individual confusion matrices. \textbf{Doctor Net} \cite{guan2018said}: it learns a neural network for every annotator's annotations and aggregates the networks' output by weighted majority voting. \textbf{Max-MIG} \cite{cao2019max}: it jointly estimates a neural classifier and a label aggregation network using an information-theoretical loss function. 

\noindent 2) Methods with complex noise models. \textbf{DL-GLAD}: it learns a neural classifier with labels aggregated by GLAD \cite{whitehill2009whose}, where annotator ability and instance difficulty are modeled. \textbf{DL-WC}: it learns a neural classifier with labels aggregated by WC \cite{imamura2018analysis}, where similar annotators are clustered to share the same confusion matrix. \textbf{AggNet} \cite{albarqouni2016aggnet}: an EM-based deep model considering annotator sensitivity and specificity.

\subsection{Experiments on Synthesized Datasets}
We evaluate the proposed method under various settings of synthesized data. Particularly, we demonstrate the effectiveness of our model with different (1)  \emph{common confusion types}; (2) \emph{common noise strength}, which is defined as the sum of off-diagonal entries in the common confusion matrix; and (3) \emph{proportion of common noise}, which reflects the percentage of annotations introduced by common confusion. 

\noindent\textbf{Datasets description}. We generate synthesized crowdsourced data on two datasets, where we directly manipulate the number of annotators and annotation generation under a variety of settings. On the \textbf{Synthetic} dataset, we completely synthesized everything. 
We first sample a mean vector for every class and then sample instance features from a multi-variate Gaussian distribution parameterized by this mean vector. In particular, we randomly generate 10,000 instances with 6 classes, which are split into a 8,000-instance training set, a 1,000-instance validation set and a 1,000-instance testing set. The \textbf{CIFAR-10} dataset is generated based on the CIFAR-10 image classification dataset \cite{krizhevsky2009learning}. It consists of 60,000 $32 \times 32$ color images from 10 classes, which are split into a 40,000-instance training set, a 10,000-instance validation set and a 10,000-instance testing set. Image features are used to train the neural classifier on this dataset. In both datasets, each instance in the training set is labeled by averaging 3 randomly selected annotators out of 30 in total.

\begin{figure*}[!htp]
\centering
\begin{subfigure}{0.24\textwidth}
  \centering
  \includegraphics[width=4cm]{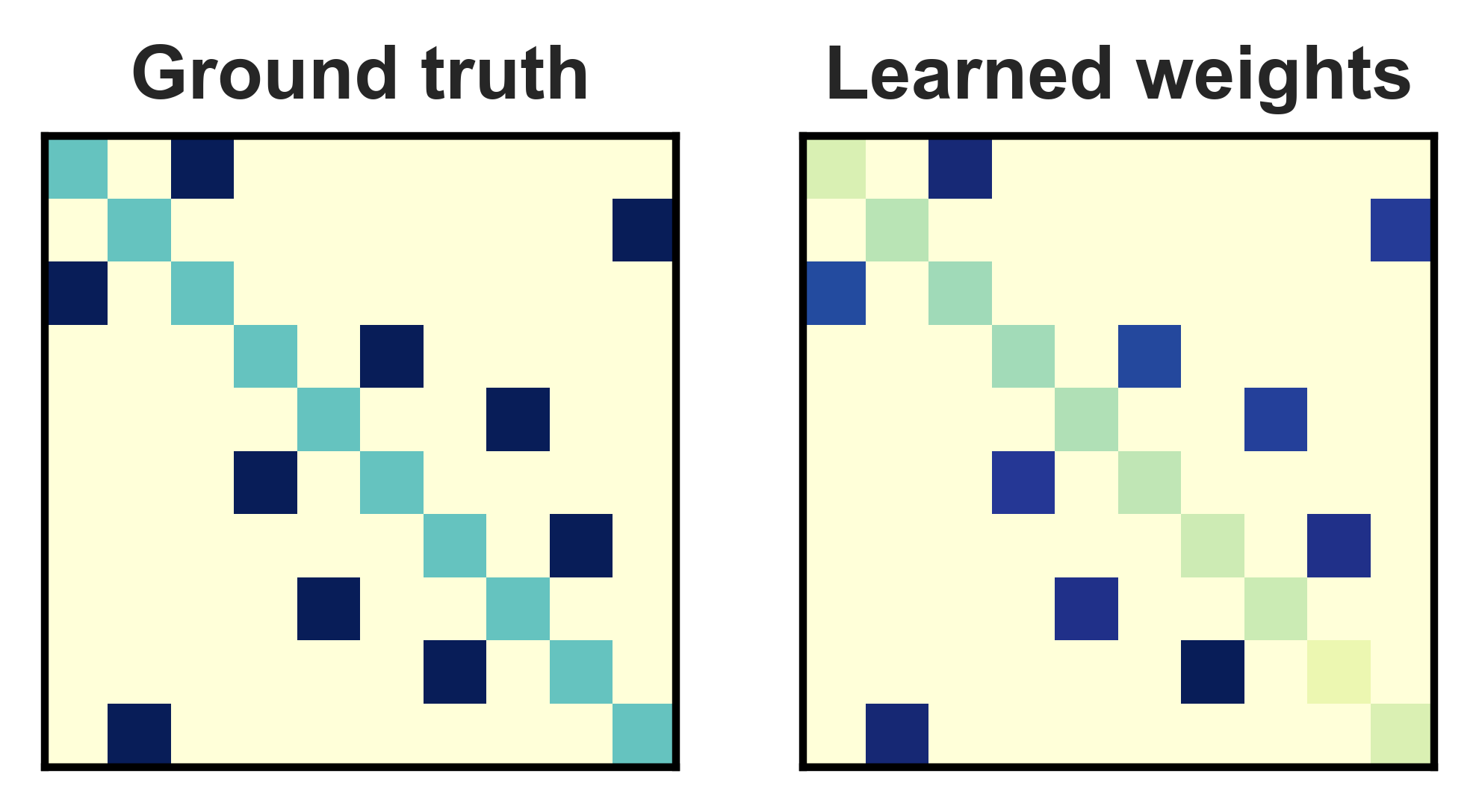}
  \caption{common noise}
\end{subfigure}
\begin{subfigure}{0.24\textwidth}
  \centering
  \includegraphics[width=4cm]{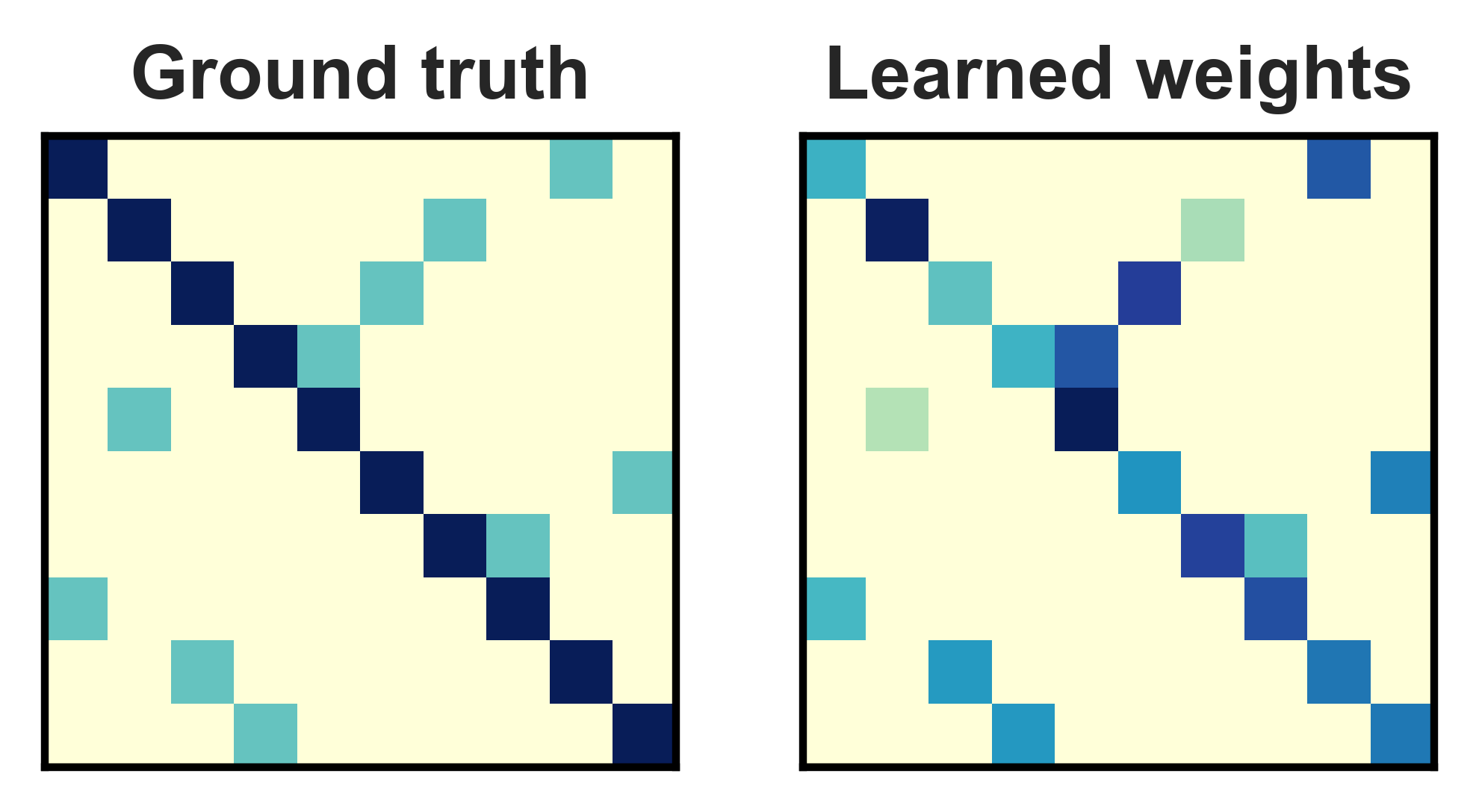}
  \caption{annotator 1}
\end{subfigure}
\begin{subfigure}{0.24\textwidth}
  \centering
  \includegraphics[width=4cm]{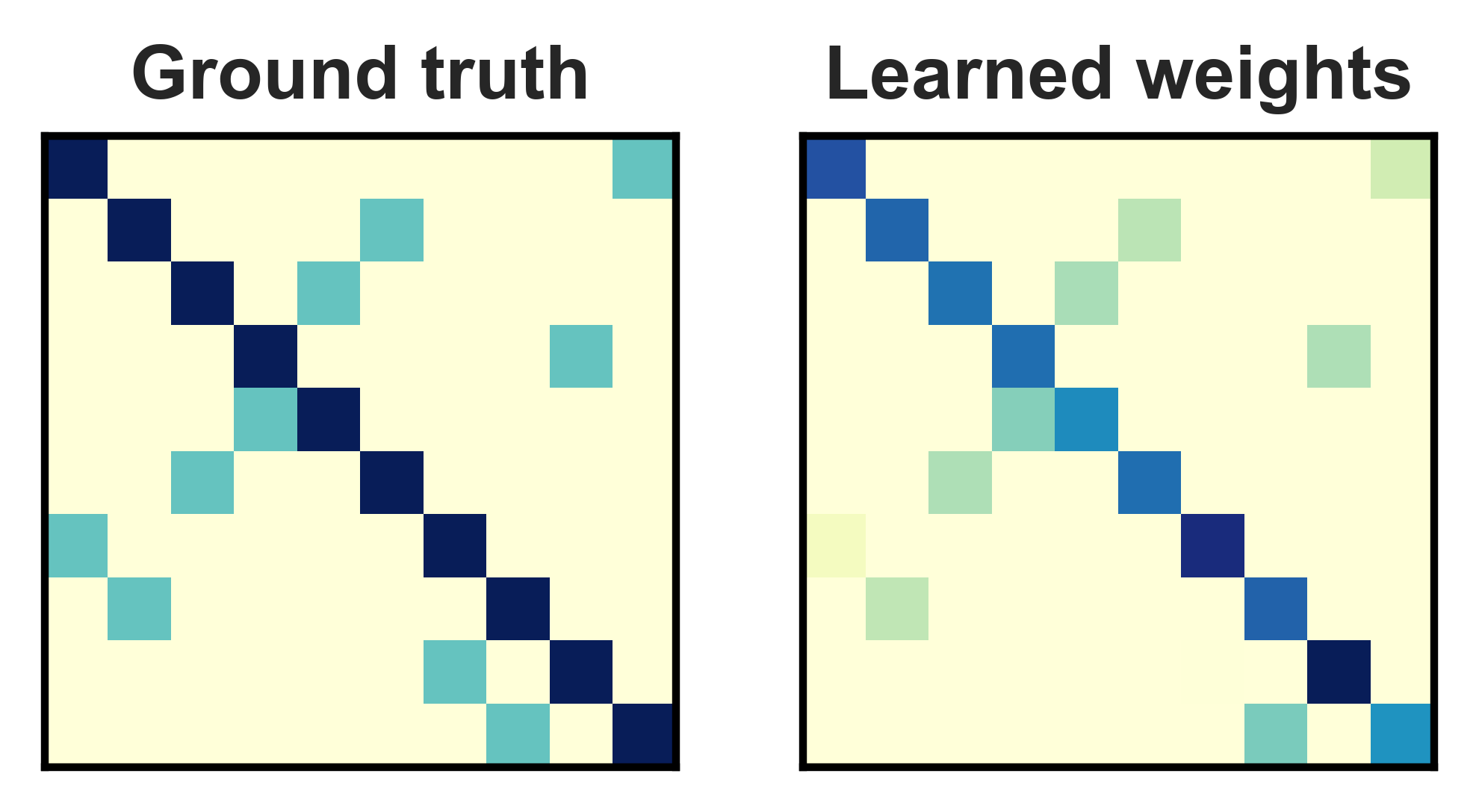}
  \caption{annotator 2}
\end{subfigure}
\begin{subfigure}{0.24\textwidth}
  \centering
  \includegraphics[width=4cm]{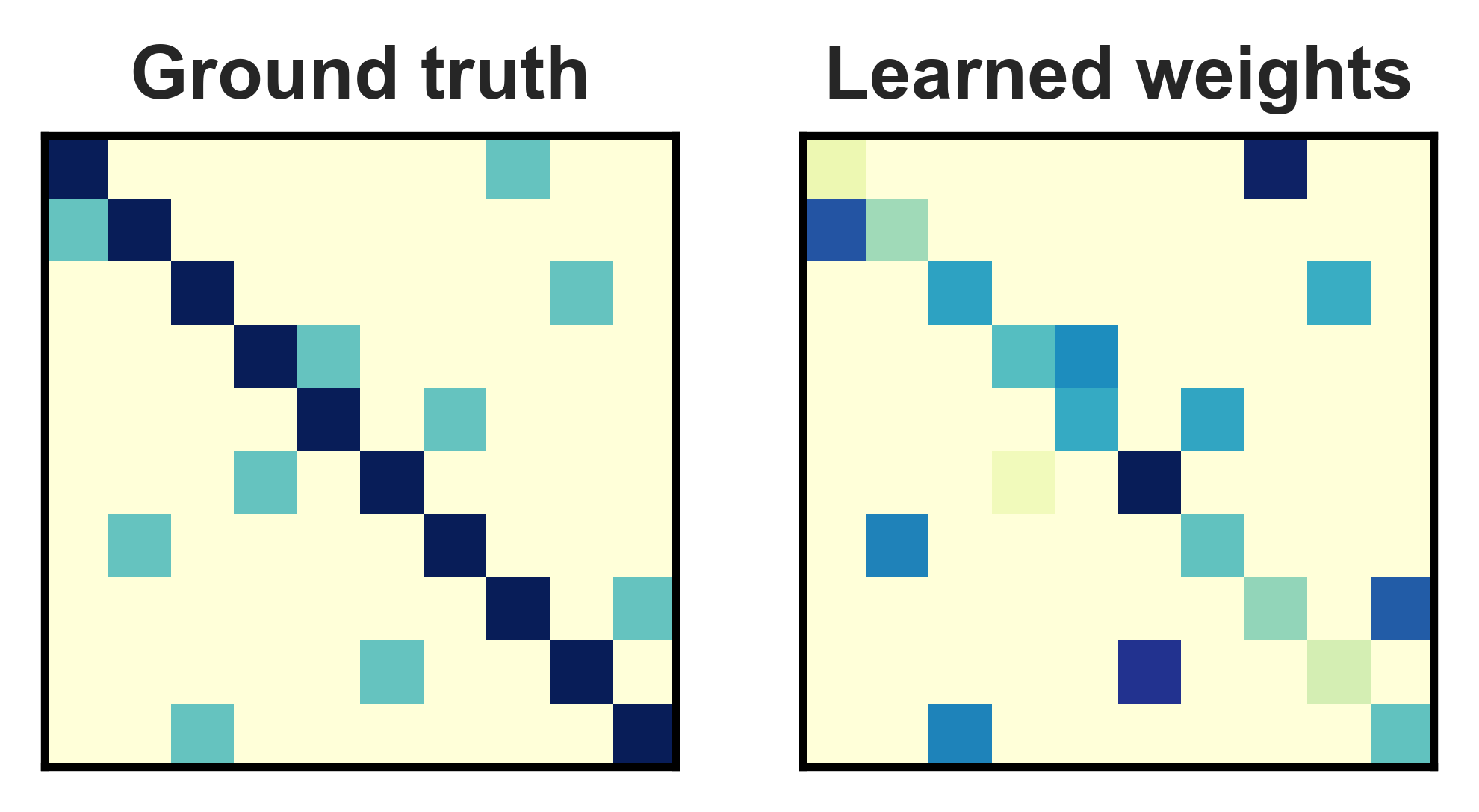}
  \caption{annotator 3}
\end{subfigure}

\begin{subfigure}{0.24\textwidth}
  \centering
  \includegraphics[width=4cm]{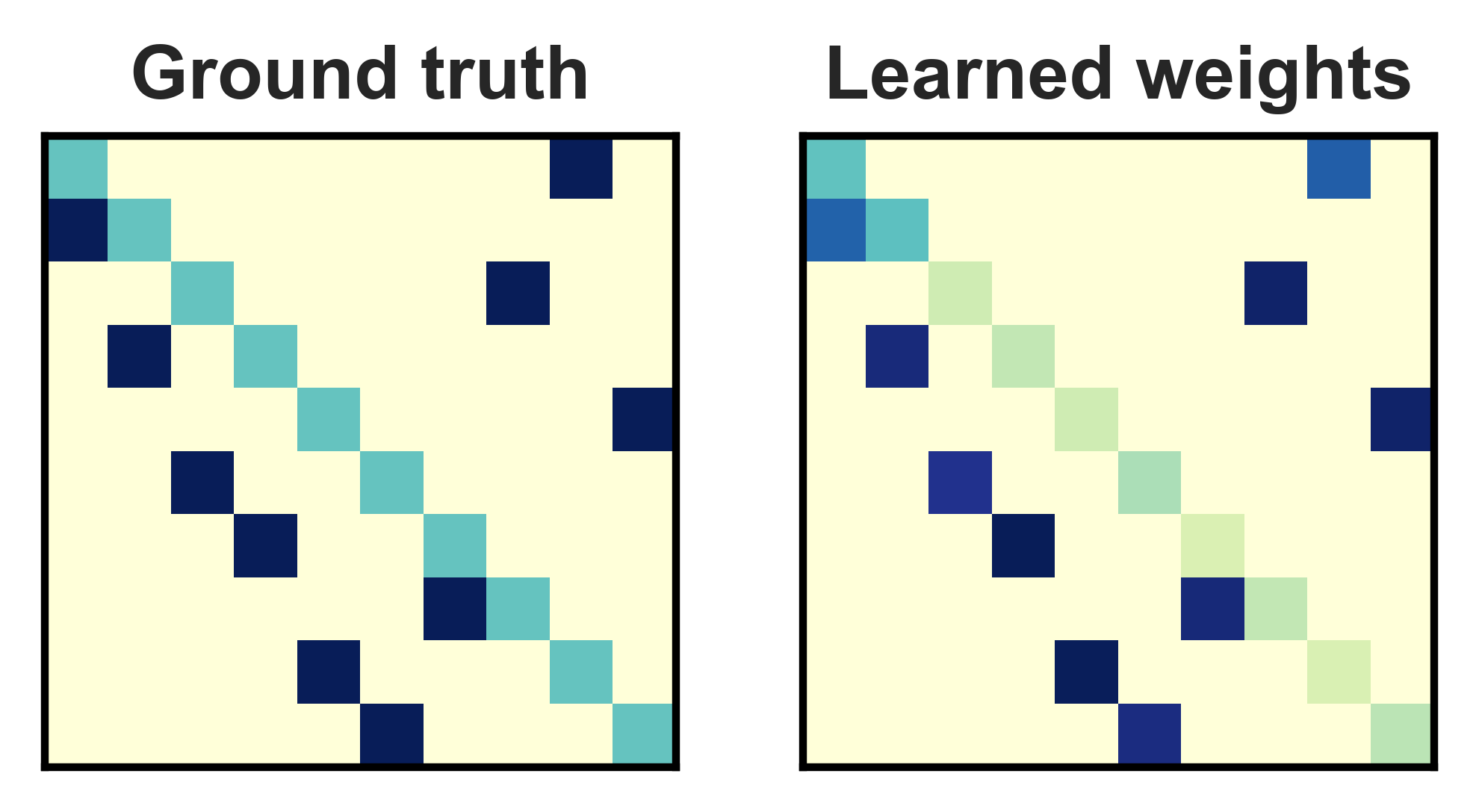}
  \caption{common noise}
\end{subfigure}
\begin{subfigure}{0.24\textwidth}
  \centering
  \includegraphics[width=4cm]{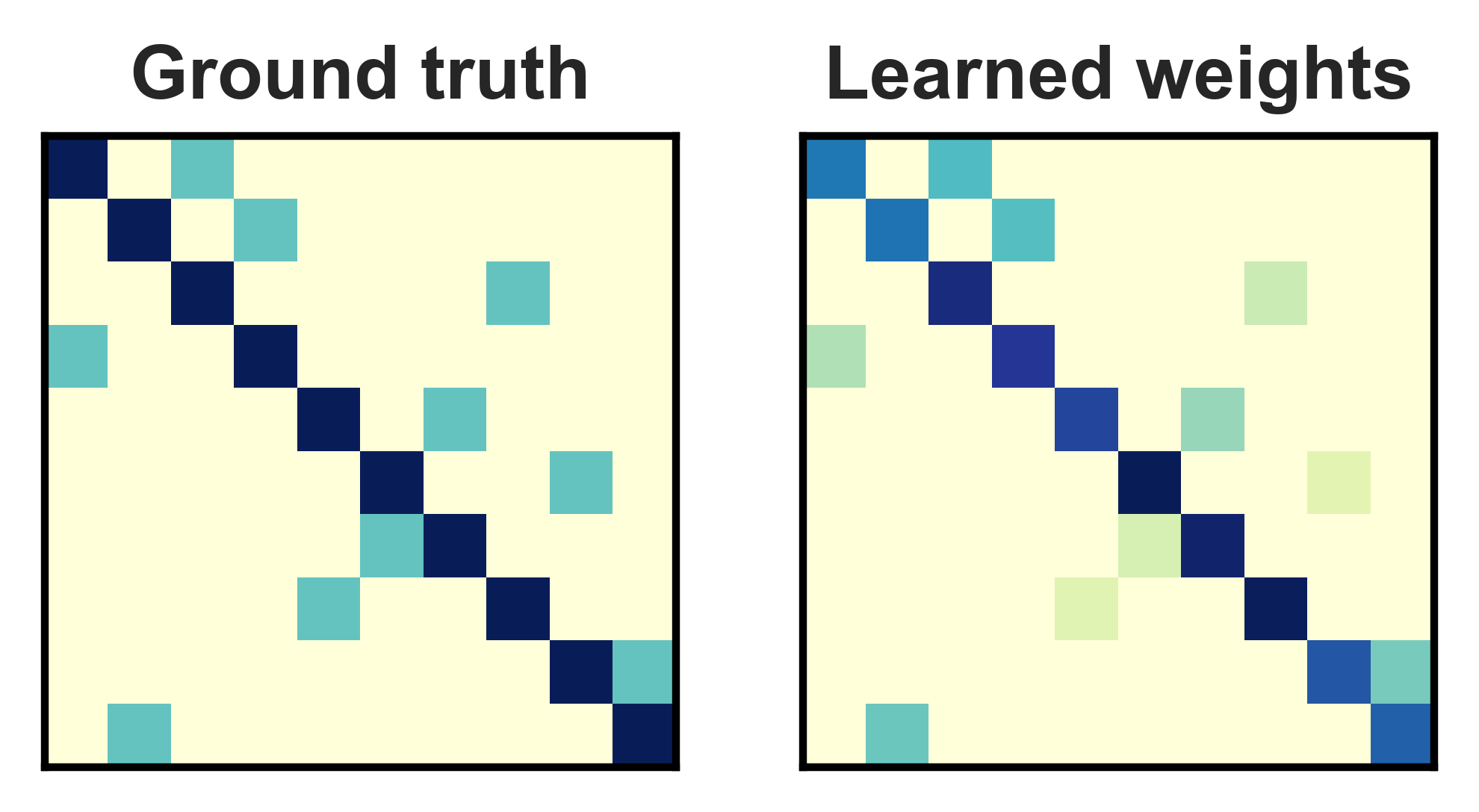}
  \caption{annotator 1}
\end{subfigure}
\begin{subfigure}{0.24\textwidth}
  \centering
  \includegraphics[width=4cm]{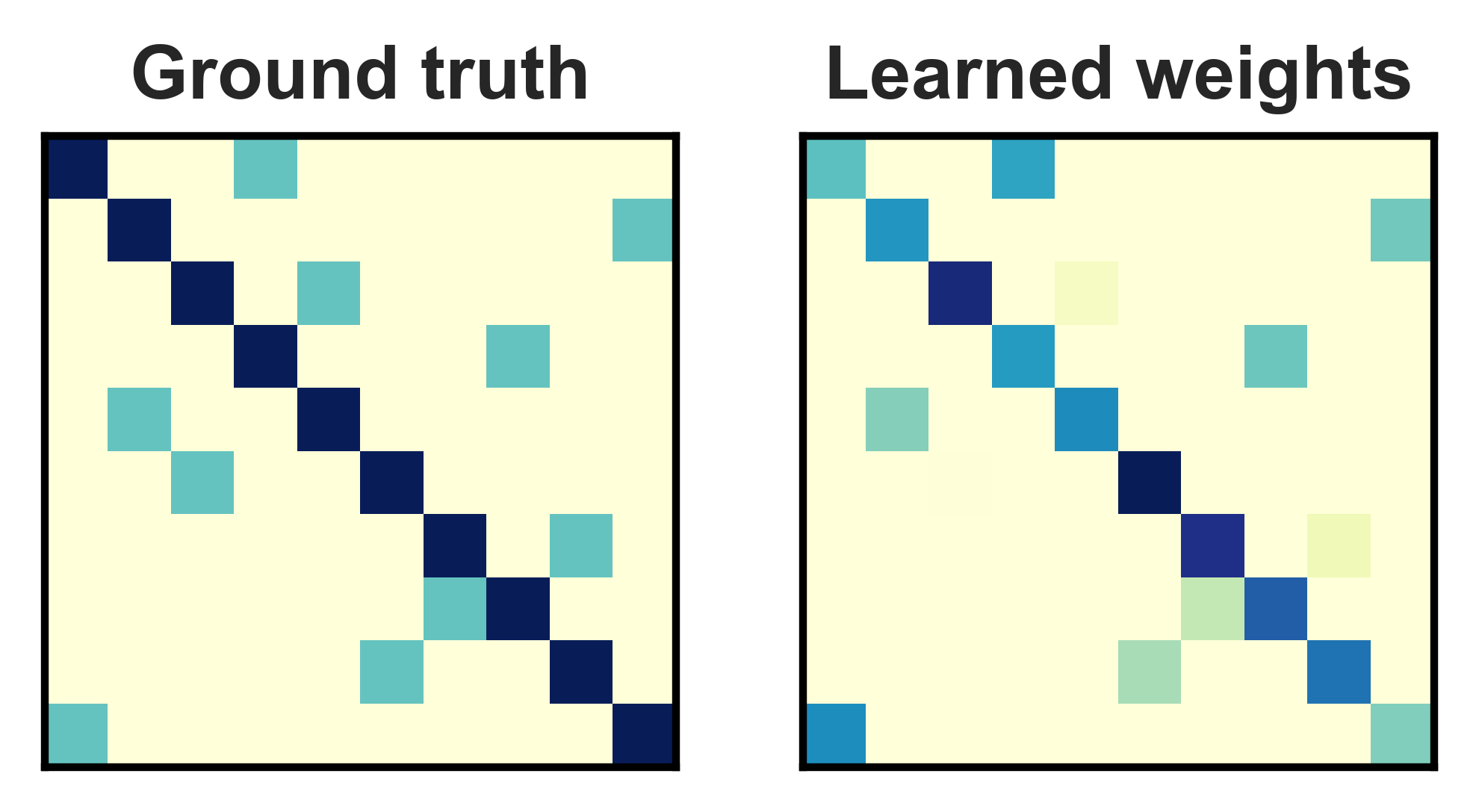}
  \caption{annotator 2}
\end{subfigure}
\begin{subfigure}{0.24\textwidth}
  \centering
  \includegraphics[width=4cm]{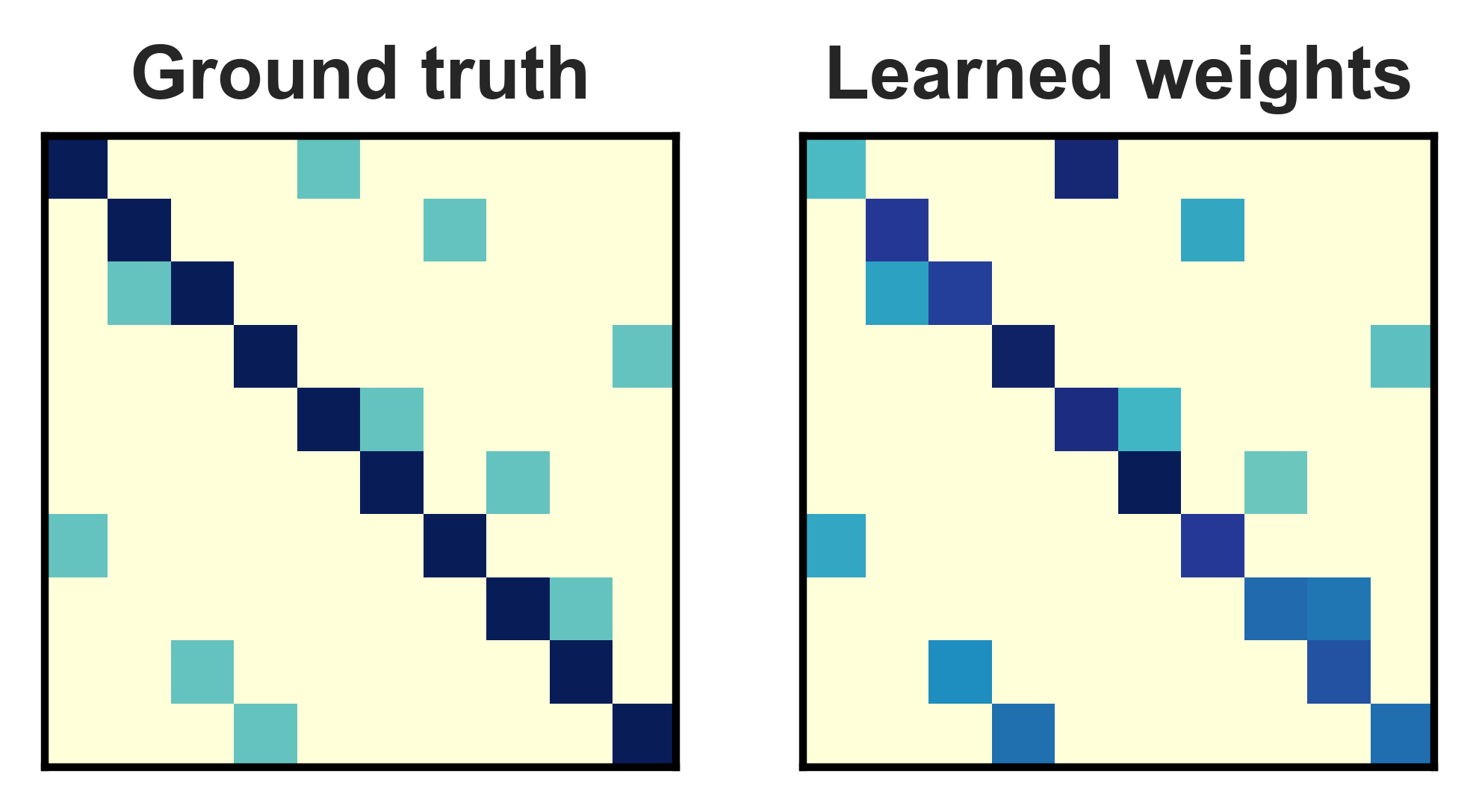}
  \caption{annotator 3}
\end{subfigure}
\caption{Comparison between ground truth confusion matrices and learned ones on CIFAR-10 dataset. The top row is the result of asymmetric common noise. The bottom row is the result of symmetric common noise.}
\label{fig:mat_recovery}
\end{figure*}

\noindent\textbf{Synthesizing annotations}. 
We consider two representative noise patterns in common noise: (1) \textit{Asymmetric confusion}. Every class is mapped to another uniformly chosen class on both datasets. (2) \textit{Symmetric confusion}. On Synthetic dataset, two random classes are paired and flipped into each other. And on CIFAR-10 dataset, we manually paired similar classes (e.g., \emph{bird} and \emph{airplane}) to be flipped with each other. For individual confusion matrices, we use asymmetric confusion. We generate one global confusion matrix, and one individual confusion matrix for every annotator. In our experiments, the range of common noise strength is set to $[0.4, 0.8]$, while the individual noise strength of annotators is fixed to 0.7. In both noise generation patterns, the noise strength is evenly distributed among the chosen off-diagonal entries.

To control the source of noise in each annotation, i.e., $s_{i}^r$, we randomly generate a set of annotator features $\boldsymbol{u}$, which are not disclosed to the learners. Given instance feature vector $\boldsymbol{v}_i$ and annotator feature vector $\boldsymbol{u}_r$, we compute $\omega_{i}^r$ by Eq \eqref{eq:proportion} with the ground-truth weight matrices $(\boldsymbol{W}_u, b_u)$ and $(\boldsymbol{W}_v, b_v)$. 
These weight matrices are not disclosed to the learner. The bias terms are used to control the average proportion of common noise across annotations into a range of $[0.3, 0.7]$. When we generate annotation $y^r_i$ for instance $i$ by annotator $r$, we first sample $s^r_i\sim B(\omega_{i}^r)$. If $s^r_i=1$, the common confusion matrix $\pi^g$ will be used; otherwise, individual confusion matrix $\pi^r$ will be used. Then we sample $y^r_i$ from the chosen confusion matrix based on the true label $z_i$ of this instance. We also include a special case that the proportion is 0, where there is no common confusion.

In our experiments, when studying the influence of common noise strength on the learnt classifier, the average proportion of common noise is controlled to be around 0.5. When studying the influence of the proportion of common noise in each annotation, the common and individual noise strength is controlled to 0.4 and 0.7 respectively. 


\noindent\textbf{Backbone networks \& training details}. On the Synthetic dataset, we apply a simple network with only one fully connected (FC) layer (with 128 units and ReLU activations), along with a softmax output layer, using 50\% dropout. On the CIFAR-10 dataset, we follow the setting of \citet{cao2019max} to use VGG-16 as the backbone network. We trained the network using the Adam optimizer \cite{kingma2014adam} with default parameters and learning rate 
searched from \{0.02, 0.01, 0.005\}. The dimension of annotator and instance embedding is chosen from \{20, 40, 60, 80\}. The regularization term $\lambda$ is searched from $\{10^{-4}, 10^{-5}, 10^{-6}\}$. All experiments are repeated 5 times with different random seeds. Model selection is achieved by choosing the model with the highest accuracy on the validation set. We report mean and standard deviation of test accuracy on the five runs.  To make the comparisons fair, all the evaluated methods used the same backbone networks. We implement our framework with PyTorch, and run it on a CentOS system with one NVIDIA 2080Ti GPU with 10 GB memory. 

\noindent\textbf{Results}. We report the results on the CIFAR-10 dataset in Figure \ref{fig:cifar_results}, where our solution demonstrated consistent improvement against all baselines across all settings. The observation on the Synthetic dataset is similar, and we present the results in Appendix C due to space limit. All the baselines assumed single source of noise, i.e., annotator-specific noise; as a result, they are heavily influenced when noise become complicated, e.g., a large proportion of mistakes from common confusion and the strength of common noise is strong. Our solution is less sensitive to the environment by decomposing and separately modeling the confusion. 
When there is no common confusion, the empirical result shows no significant difference between our solution and baselines in this extreme setting, which should also be expected. But we argue that this extreme setting rarely holds in reality, as annotators always share some commonsense about the world.

\begin{table*}[!htp]
\centering
\begin{tabular}{c @{\hspace{0.4\tabcolsep}} |c @{\hspace{0.4\tabcolsep}} c @{\hspace{0.4\tabcolsep}}c @{\hspace{0.4\tabcolsep}}c @{\hspace{0.4\tabcolsep}}c @{\hspace{0.4\tabcolsep}}|c @{\hspace{0.4\tabcolsep}}c @{\hspace{0.4\tabcolsep}} c @{\hspace{0.4\tabcolsep}}|c}
\toprule
        & DL-MV & DL-CL & Doctor Net & Anno-Reg & Max-MIG &DL-GLAD & DL-WC & AggNet & CoNAL \\ \midrule
LabelMe & 79.83{\footnotesize $\pm 0.34$} & 83.27{\footnotesize $\pm 0.52$}    & 82.12{\footnotesize $\pm 0.43$} & 82.77{\footnotesize $\pm 0.48$}  & 85.33{\footnotesize $\pm 0.61$}  & 83.12{\footnotesize $\pm 0.34$} & 82.74{\footnotesize $\pm 0.33$} & 84.75{\footnotesize $\pm 0.27$}& \textbf{87.12}{\footnotesize $\pm 0.55$}\\ \midrule
Music   & 72.53{\footnotesize $\pm 0.41$}  & 81.46{\footnotesize $\pm 0.53$}      & 76.58{\footnotesize $\pm 0.47$}       & 79.12{\footnotesize $\pm 0.36$} & 81.37{\footnotesize $\pm 0.33$} & 77.82{\footnotesize $\pm 0.37$} & 75.76{\footnotesize $\pm 0.24$} & 81.92{\footnotesize $\pm 0.41$} & \textbf{84.06}{\footnotesize $\pm 0.42$} \\ \bottomrule
\end{tabular}
\caption{Test accuracy on two real-world crowdsourcing datasets.}
\label{table:real}
\end{table*}

All models are influenced by symmetric common noise, which directly makes the swapped classes similar. Based on the lower bound provided in Theorem \ref{trm1LB}, similar conditional class distributions in the confusion matrices will make the problem more difficult, so that the degeneration of all methods are expected under symmetric confusion. In the most extreme case where the proportion of common noise is set to 0.7 and the common noise strength is set to 0.6, nearly 42\% annotations are pairwise flipped. However, our method can still outperform baselines with a large margin. Mix-MIG is believed to be robust to correlated mistakes if high-quality annotator exists. However, our experiments show that common confusion poisoned the classifier obtained in Max-MIG even though every annotator is of high quality (individual noise strength is set to 0.7). DL-CL and Anno-Reg failed because they could not differentiate the source of noise, such that the gradients from the modeled annotations cannot be properly adjusted to update the classifier. Both Doctor Net and DL-MV are based on majority vote, so that they fail when the annotations across annotators are no longer independent, i.e., caused by the common confusion. Compared to methods with complex noise models, DL-GLAD directly models the annotation accuracy, which is not suitable for class-dependent confusion. DL-WC clusters correlated annotators to share confusion matrix, which can reduce the influence of common confusion. But the expertise of each annotator is missing, which leads to its bad performance. AggNet shows the advantage of directly learning from annotations rather than from aggregated labels. But it still assumes the only noise source thus cannot handle common noise well.

To understand how accurate our solution can distinguish common and individual noise, we report the learnt weights of noise adaptation layers against the ground-truth confusion matrices on the CIFAR-10 dataset in Figure \ref{fig:mat_recovery}. In this experiment, we set the common noise strength to 0.7 and the proportion of common noise to 0.5. We can find that in most cases the ground-truth common noise pattern is well recovered, especially under the asymmetric noise pattern. 



\subsection{Experiments on Real-world Datasets}

\noindent\textbf{Datasets description}. We consider two real-world datasets. \textbf{LabelMe} \cite{rodrigues2018deep, russell2008labelme} is an image classification dataset, consists of 2,688 images from 8 classes, where 1,000 of them are labeled by annotators from Amazon Mechanical Turk (AMT)\footnote{https://www.mturk.com/} and the remainings are used for validation and testing. Each image is labeled by an average of 2.5 annotators, with a mean accuracy of 69.2\%. Standard data augmentation techniques are used on training data, including horizontal flips, rescaling and shearing, to enrich the training set to 10,000 images. \textbf{Music} \cite{rodrigues2014gaussian} is a music genre classification dataset, consisting of 1,000 samples of songs with 30 seconds length from 10 music genres, where 700 of them are labeled by AMT annotators and the rest are used for testing. 
Each sample is labeled by an average of 4.2 annotators, with a mean annotation accuracy of 73.2\%.

\noindent\textbf{Backbone networks \& training details}. For LabelMe dataset, we followed the setting of \citet{rodrigues2018deep}: we apply a pre-trained VGG-16 network followed by a FC layer with 128 units and ReLU activations, and a softmax output layer, using 50\% dropout. For Music dataset, we use the same FC layer and softmax layer as LabelMe. Batch normalization \cite{ioffe2015batch} is performed in each layer. Other hyper-parameters are the same as the synthesized experiments.

\noindent\textbf{Results}. As reported in Table \ref{table:real}, CoNAL achieved new state-of-the-art performance on both real-world datasets. In particular, we looked into the accuracy on classes where commonly made mistakes across annotators are observed (see Figure \ref{fig:noise_anlysis}). For example, for \emph{open country} on LabelMe, its accuracy in CoNAL is 67.21\%, while the best baseline Max-MIG only achieved 54.19\%. The good performance aligns with our analysis in Theorem \ref{trm1LB}, by differentiating common and individual confusions, it is easier to find the true labels. We provide the visualization of the learned confusion matrices and the training and testing accuracy plots on real-world datasets in Appendix C. 

\noindent\textbf{Influence of the regularization term $\lambda$}. We studied the influence of different $\lambda$ in Table \ref{table:lambda}. The results show by enforcing the noise adaptation layers to be different, the performance is improved on both datasets. The value of $\lambda$ also matters, and $10^{-5}$ achieves best performance empirically.

\begin{table}[!htp]
\centering
\begin{tabular}{c @{\hspace{0.4\tabcolsep}} c @{\hspace{0.4\tabcolsep}} c @{\hspace{0.4\tabcolsep}} c @{\hspace{0.4\tabcolsep}}c @{\hspace{0.4\tabcolsep}}}
\toprule
 $\lambda$ & 0 & $10^{-4}$ & $10^{-5}$  & $10^{-6}$  \\ \midrule
LabelMe & 85.68{\footnotesize $\pm 0.38$} & 86.61{\footnotesize $\pm 0.41$}   & 87.12{\footnotesize $\pm 0.55$}  & 86.26{\footnotesize $\pm 0.47$}  \\ \midrule
Music   & 82.14{\footnotesize $\pm 0.31$}  & 83.52{\footnotesize $\pm 0.25$}  & 84.06{\footnotesize $\pm 0.42$}  & 82.98{\footnotesize $\pm 0.37$}  \\ \bottomrule
\end{tabular}
\caption{Model performance under different $\lambda$.}
\label{table:lambda}
\end{table}

%% file: broader_impact.tex
\section*{Ethics statement}
Our study focuses on tackling an urgent problem in this deep learning era: learning from crowds. High-quality labels are needed for real-world deep learning applications; however, they are typically difficult and expensive to collect in practice. Hence, we propose to directly learn from labels given by non-expert annotators, considering both common mistakes and individualized mistakes. On the one hand, industrial applications will benefit from this work since non-expert labels are both cost- and time-effective to enable deployment of deep learning systems. On the other hand, our work also has academic impact. Our method can be applied to new research problems where high-quality labeled data is rare but crowdsourced labels are easy to obtain, such as medical image classification. 

The potential issue of common noise modeling is it might open the door for adversarial annotators. When previously modeled independently, they need to provide a large number of annotations to  poison a learner. But if an attacker gets access to common noise, he/she only needs to provide a few annotations consistent with the common noise to amplify the influence of common noise. This will also make other ordinary annotators inadvertently contribute to the attack. Another potential issue of learning from crowds is when modeling annotator expertise, we are learning an annotator profile, which has risk in disclosing their privacy, especially in privacy sensitive annotation problems. Data masking or distortion (e.g., differential privacy) is needed to protect annotators’ privacy.

%% file: appendix.tex
\appendix
\clearpage
\section{A. Proof of Theorem 1}
\label{app:proof}

\begin{proof} In our setting, the ground-truth class distribution $\rho_i$ depends on the instance features. Then the minimax error rate of the crowdsourcing problem can be lower bounded by the following,
\begin{equation}
    \begin{aligned}
    \textup{inf}_{\hat{Z}} \textup{sup}_{Z\in [C]^N} \mathbb{E}\left[\mathcal{L}(\hat{Z}, Z)\right] \geq & \frac{1}{N^2\textup{log} C}\sum_{i=1}^NR(\rho_i, \Pi')  \\
    &- \frac{\textup{log} 2}{N^2\textup{log} C} 
    \label{eq:lb_ds}
    \end{aligned}
\end{equation}
where
\begin{equation}
    \label{eq:r}
    R(\rho_i, \Pi') = H(\rho_i) - \sum_{r=1}^{R}\sum_{c=1}^C\sum_{c'=1}^C\rho_{ic}\rho_{ic'}\textup{KL}(\pi'^r_{c*} \parallel \pi'^r_{c'*})
\end{equation}
and $\Pi' = \{\pi'^r\}_{r=1}^R$ denotes the set of annotator-level confusion matrices. We use $\pi'$ to differentiate with our defined individual confusion matrix in the main paper. The proof of Eq \eqref{eq:lb_ds} is similar to \cite{imamura2018analysis}. Based on our new noise generation assumption, the annotation noise can be decomposed by common noise and individual noise. Thus we can further bound the minimax error rate under this noise assumption.

Under our new noise assumption, we can evaluate the confusion matrix on a per-instance-annotator basis. Specifically, in each annotation, the effective confusion matrix is a weighted combination of the global and individual confusion matrices, where the weight is $w^r_i$. In a mixture model, the Kullback–Leibler divergence can be decomposed accordingly by,

\begin{align}
        \textup{KL}(\pi'^r_{c*} \parallel \pi'^r_{c'*}) = & \; \textup{KL}(\omega_i^r \pi^g_{c*} + (1-\omega_i^r)\pi^r_{c*} \parallel  \nonumber \\
        & \quad\quad\quad\quad\quad\quad \omega_i^r \pi^g_{c'*} + (1-\omega_i^r)\pi^r_{c'*}) \nonumber \\
    \leq & \; \textup{KL}(\boldsymbol{\omega}_i^r \parallel \boldsymbol{\omega}_i^r )+ \omega_i^r\, \textup{KL}(\pi^g_{c*} \parallel \pi^g_{c'*}) \nonumber \\
    & + (1-\omega_i^r) \,\textup{KL}(\pi^r_{c*}\parallel \pi^r_{c'*}) \label{eq:log_sum} \\
    =&\; \omega_i^r\, \textup{KL}(\pi^g_{c*} \parallel \pi^g_{c'*}) \nonumber \\
     &+ (1-\omega_i^r) \,\textup{KL}(\pi^r_{c*}\parallel \pi^r_{c'*})\label{eq:ineq}
\end{align}

where $\boldsymbol{\omega}_i^r = (\omega_i^r, 1-\omega_i^r)$. The inequality can be derived by the log-sum inequality. Substitute Eq \eqref{eq:ineq} back to Eq \eqref{eq:r}, we can get the new term $F(\rho, \Pi, \Omega)$ in Theorem 1. Plug it back into Eq \eqref{eq:lb_ds}, we can get the refined result in our Theorem 1,

\begin{equation*}
\begin{aligned}
    \textup{inf}_{\hat{Z}} \textup{sup}_{Z\in [C]^N} \mathbb{E}\left[\mathcal{L}(\hat{Z}, Z)\right] \geq& \frac{1}{N^2\textup{log}\,C}\sum_{i=1}^NF(\rho_i, \Pi, \Omega) \\
    &- \frac{\textup{log}\,2}{N^2\textup{log}\,C}
\end{aligned}
\end{equation*}
\end{proof}

\begin{corollary}
When $N \geq \frac{2\textup{log}2}{\textup{max}F(\rho_i, \Pi, \Omega)}$, increasing the number of instances $N$ will decrease the error rate bound.
\end{corollary}
\begin{proof}
\begin{align*}
    &\frac{1}{N^2\textup{log}\,C}\sum_{i=1}^NF(\rho_i, \Pi, \Omega) - \frac{\textup{log}\,2}{N^2\textup{log}\,C} \\
    \leq& \frac{\textup{max}F(\rho_i, \Pi, \Omega)}{N\textup{log}\,C} - \frac{\textup{log}\,2}{N^2\textup{log}\,C},
\end{align*}

When the gradient of the upper bound is less than 0, the upper bound will decrease when $N$ is growing. This can be achieved by setting $N$ by the following,

\begin{gather*}
    -\frac{1}{N^2}\frac{\textup{max}F(\rho_i, \Pi, \Omega)}{\textup{log}\,C} + \frac{2\textup{log}2}{N^3\textup{log}\,C} \leq 0 \\
    \Rightarrow N \geq \frac{2\textup{log}2}{\textup{max}F(\rho_i, \Pi, \Omega)}
\end{gather*}
\end{proof}
\noindent\textbf{Remarks}. The corollary shows when the number of instances is growing, the label aggregation quality gets improved. Also, we need to point out the structure of confusion matrices $\Pi$ is more important than the number of classes $C$ in this lower bound. With a larger KL distance between every pair of rows in $\Pi$, we can expect an improved error lower bound.

 \begin{figure*}[!htp]
\centering
\begin{subfigure}{1\textwidth}
  \centering
  \includegraphics[width=15.1cm]{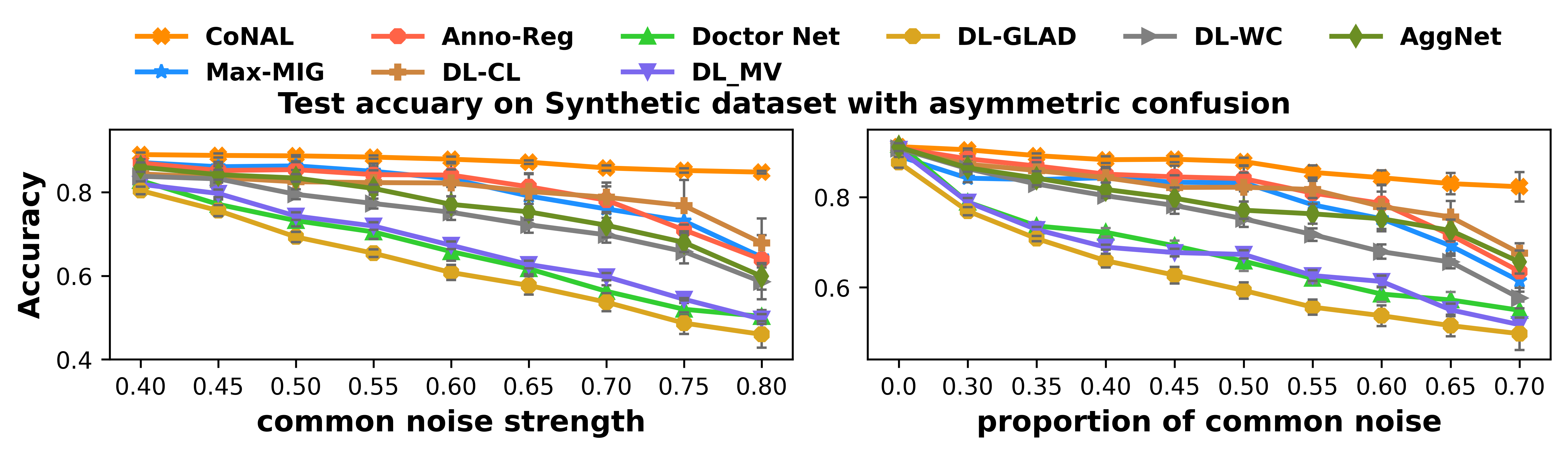}
  \label{fig:sub1}
\end{subfigure}

\begin{subfigure}{1\textwidth}
  \centering
  \includegraphics[width=15.1cm]{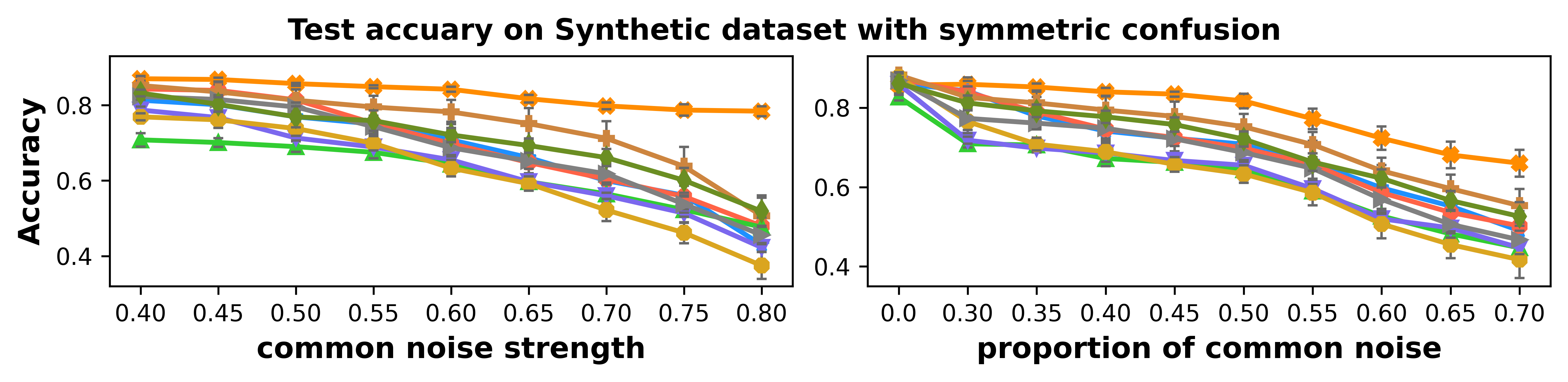}
  \label{fig:sub2}
\end{subfigure}
\caption{Results on Synthetic dataset.}
\label{fig:syn_results}
\vspace{-1.2em}
\end{figure*}

\section{B. EM algorithm for learning from crowds by modeling common confusions}
\label{app:em}
The EM algorithm is a generic solution for aggregating crowdsourced labels in classic crowdsourcing problems \cite{dawid1979maximum, imamura2018analysis}, and it can also be used under our common confusion assumption. Though we have pointed out the main drawbacks of EM-based algorithms in our solution framework, we still list the procedures of using EM algorithm in our problem for interested readers. In particular, we demonstrate a two-step solution, where the latent indicator $s_i^r$ is drawn from a Bernoulli distribution directly parameterized by $\omega_i^r$ and the ground-truth label $z_i$ is drawn from a multinomial distribution $p_{\boldsymbol{\theta}}(z_i|\textbf{x}_i)$ parameterized by $\boldsymbol{\theta}$, which is essentially the soft-classifier we are estimating from the crowdsourced data. Once the ground-truth labels $\{z_i\}^N_{i=1}$ on instances are inferred, we estimate the parameters $\boldsymbol{\theta}$ in $p_{\boldsymbol{\theta}}(z_i|\textbf{x}_i)$ by treating the inferred labels as ground-truth. When the instance features are unavailable, we can use another multinomial distribution $p(z_i|\rho)$ to replace $p_{\boldsymbol{\theta}}(z_i|\textbf{x}_i)$, where $\rho = \{\rho_c\}_{c=1}^C$ is the corresponding Dirichelet prior, to perform answer aggregation by EM as well.

Under our common confusion assumption, the conditional probability $p(y_i^r|z_i)$ can be written as

\begin{equation*}
    p(y_i^r|z_i; \Pi, \omega^r_i) = \sum_{s^r_i=\{0,1\}}p(s^r_i|w^r_i)p(y_i^r|z_i,s^r_i, \pi^r)
\end{equation*}

Based on this conditional probability, we derive the EM procedure to infer the ground-truth labels as follows. In the E-step, we estimate hidden ground-truth label $z_i$ and latent indicator $s_i^r$. The posterior $q(z_i)$ and $q(s_i^r)$ are obtained using Bayes' rule,

\begin{gather}
     q(z_i=c) \propto p_{\boldsymbol{\theta}_0}(z_i=c|\textbf{x}_i)\prod_{r=1}^Rp(y_i^r|z_i=c; \Pi_0, \omega_{i0}^r), \nonumber \\
     q(s_i^r=1) \propto \omega^r_{i0}\,p(y_i^r|z_i, \pi^g_0), \nonumber\\
     q(s_i^r=0) \propto (1-\omega^r_{i0})\,p(y_i^r|z_i, \pi^r_0). \nonumber
\end{gather}

\noindent where $\boldsymbol{\theta}_0, \omega_0$ and $\Pi_0$ are the current estimated parameters. In the M-step, we update the parameters of neural network $\boldsymbol{\theta}$, proportion of common noise $\omega$ and confusion matrices $\Pi$. The proportion of common noise and confusion matrices have closed-form solutions by using the Lagrange multiplier method \cite{bertsekas2014constrained}, 

\begin{gather*}
    \omega^r_i = q(s_i^r=1) \\
    \pi^g_{c, l} =\frac{\sum_{i=1}^N\sum_{r=1}^Rq(z_i=c)q(s_i^r=1)\mathbb{I}(y^r_i=l)}{\sum_{i=1}^N\sum_{r=1}^Rq(z_i=c)q(s_i^r=1)}, \\
    \pi^r_{c, l} = \frac{\sum_{i=1}^Nq(z_i=c)q(s_i^r=0)\mathbb{I}(y^r_i=l)}{\sum_{i=1}^Nq(z_i=c)q(s_i^r=0)}
\end{gather*}

To update the neural network parameter $\boldsymbol{\theta}$, we follow the approach in \cite{goldberger2016training, albarqouni2016aggnet} and use the inferred posterior of ground-truth $q(z_i)$ as the target. Specifically, we compute the cross-entropy loss and backpropagate the error using stochastic gradient optimization techniques such as Adam \cite{kingma2014adam}. For the generic setting where instance features are unavailable, we can update the class distribution $\rho$ using its closed-form solution, 
 
 \begin{equation*}
     \rho_c = \frac{1}{N}\sum_{i=1}^Nq(z_i=c)
 \end{equation*}

\begin{figure*}[!htp]
\centering
\begin{subfigure}{0.24\textwidth}
  \centering
  \includegraphics[width=4cm]{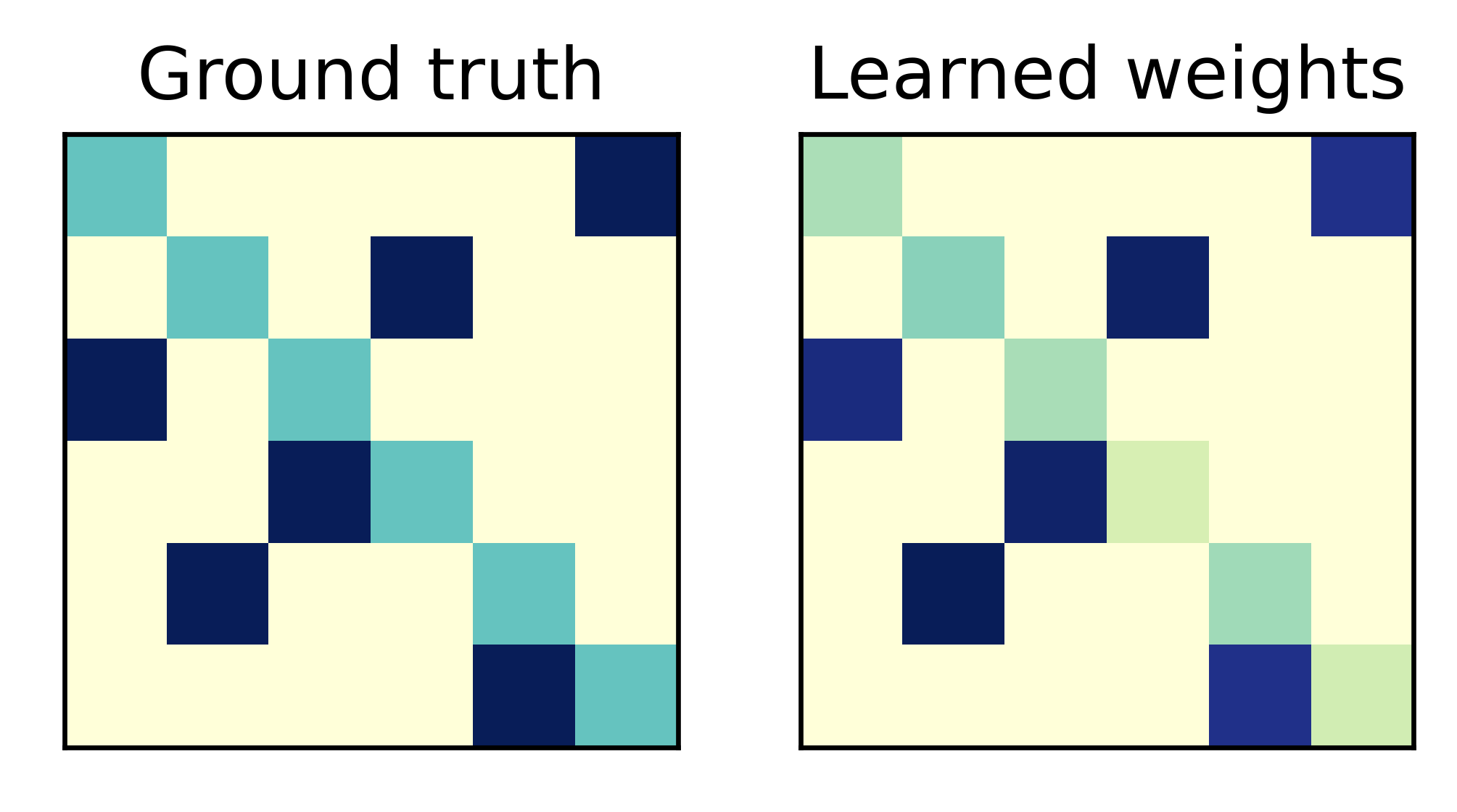}
  \caption{common noise}
\end{subfigure}
\begin{subfigure}{0.24\textwidth}
  \centering
  \includegraphics[width=4cm]{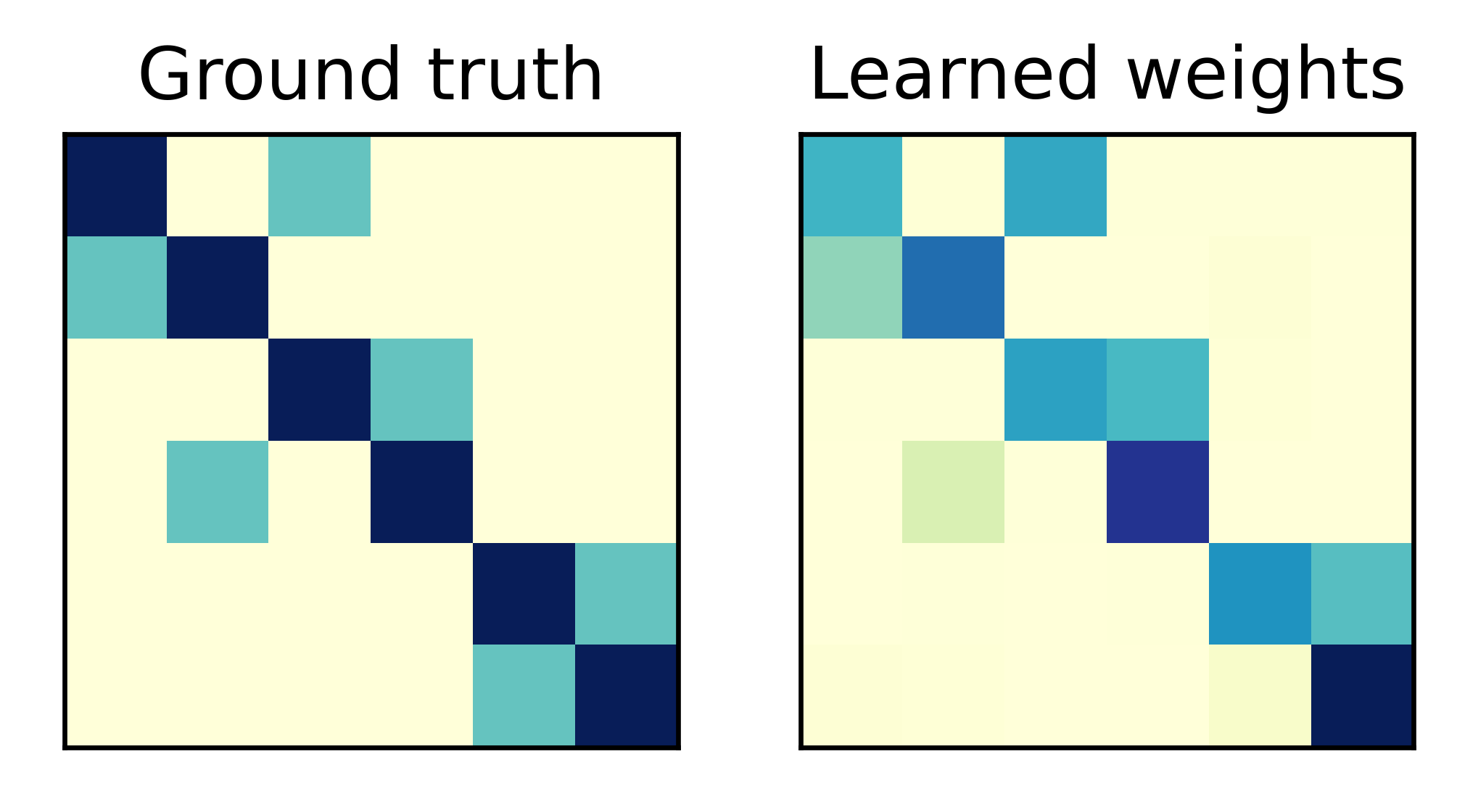}
  \caption{annotator 1}
\end{subfigure}
\begin{subfigure}{0.24\textwidth}
  \centering
  \includegraphics[width=4cm]{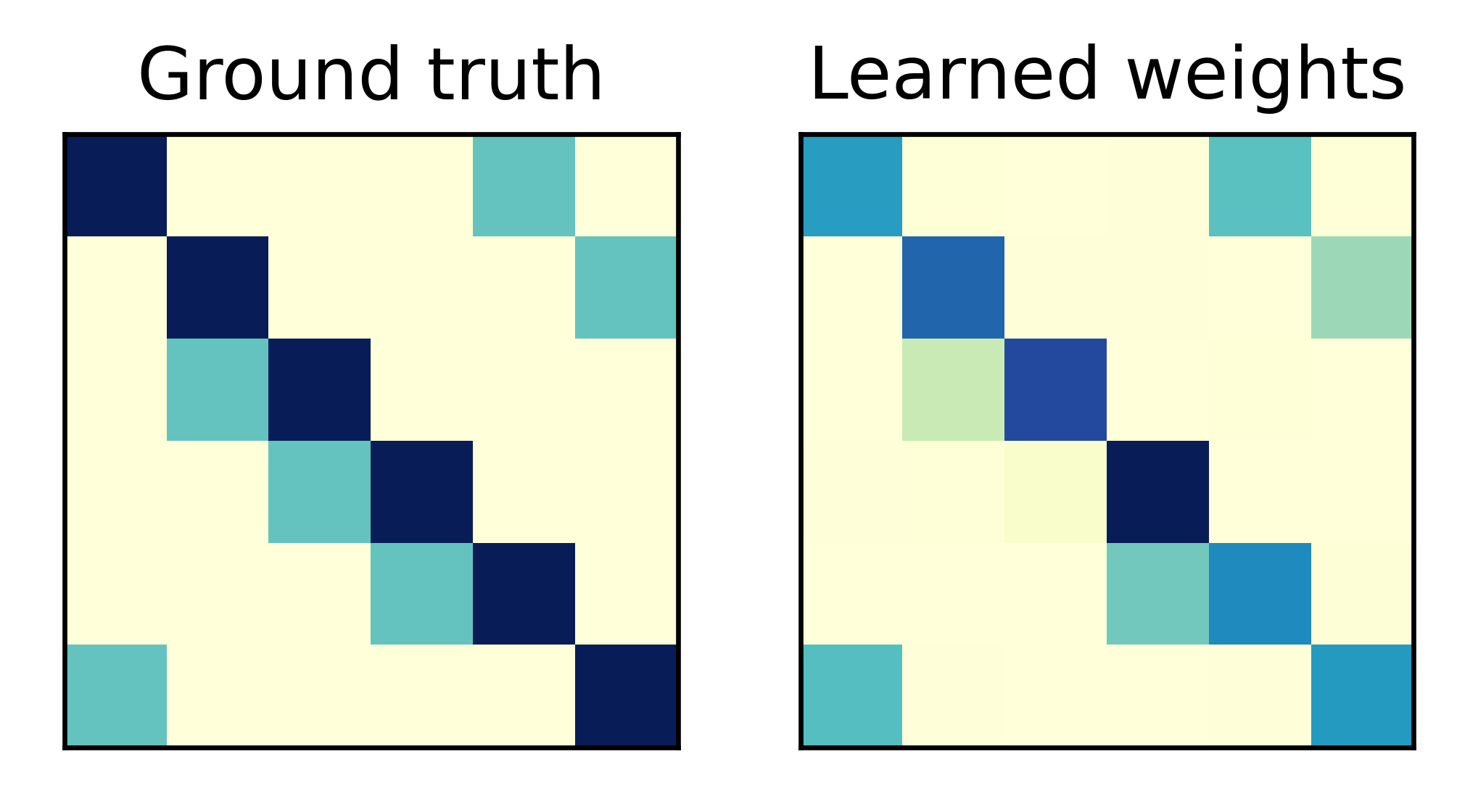}
  \caption{annotator 2}
\end{subfigure}
\begin{subfigure}{0.24\textwidth}
  \centering
  \includegraphics[width=4cm]{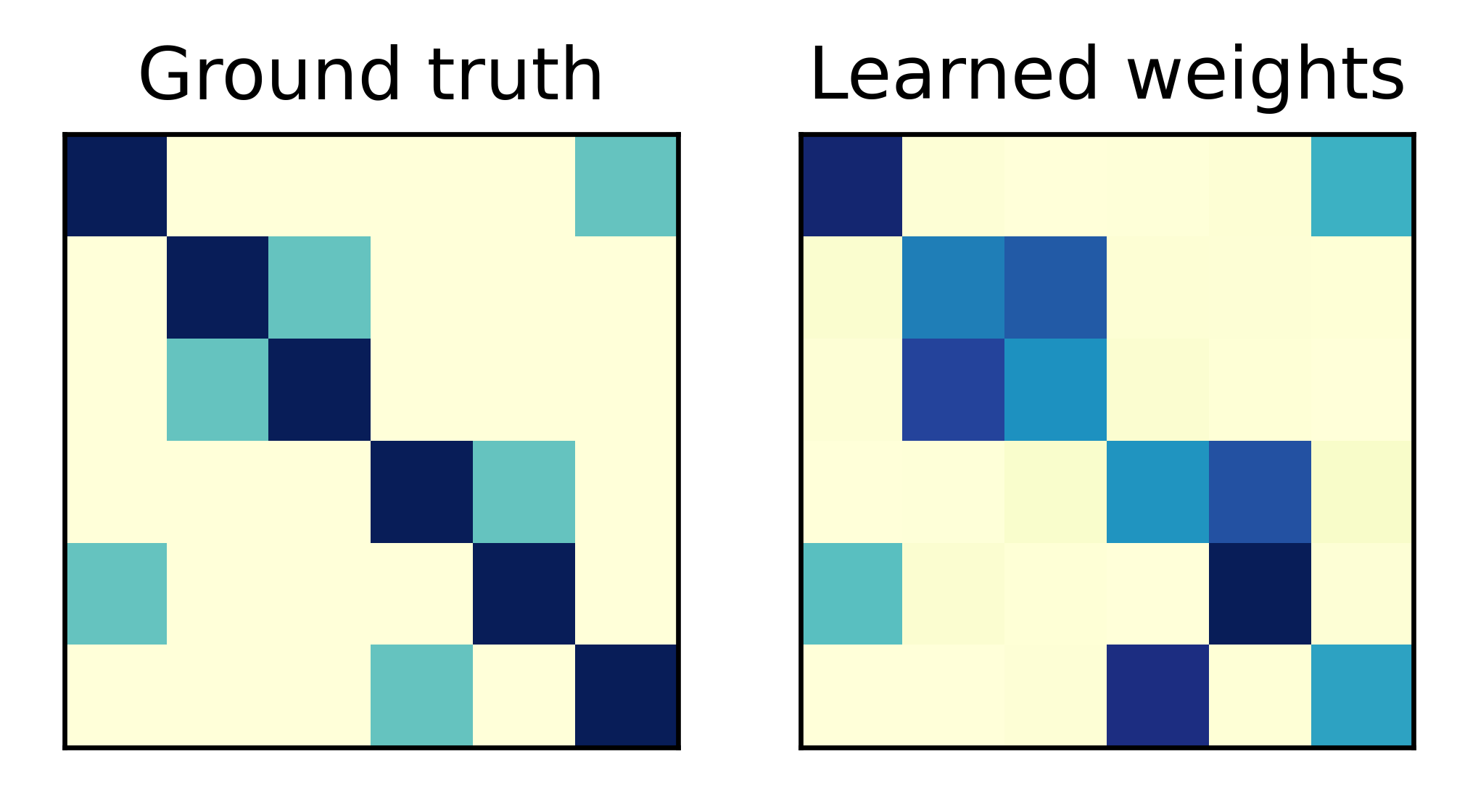}
  \caption{annotator 3}
\end{subfigure}

\begin{subfigure}{0.24\textwidth}
  \centering
  \includegraphics[width=4cm]{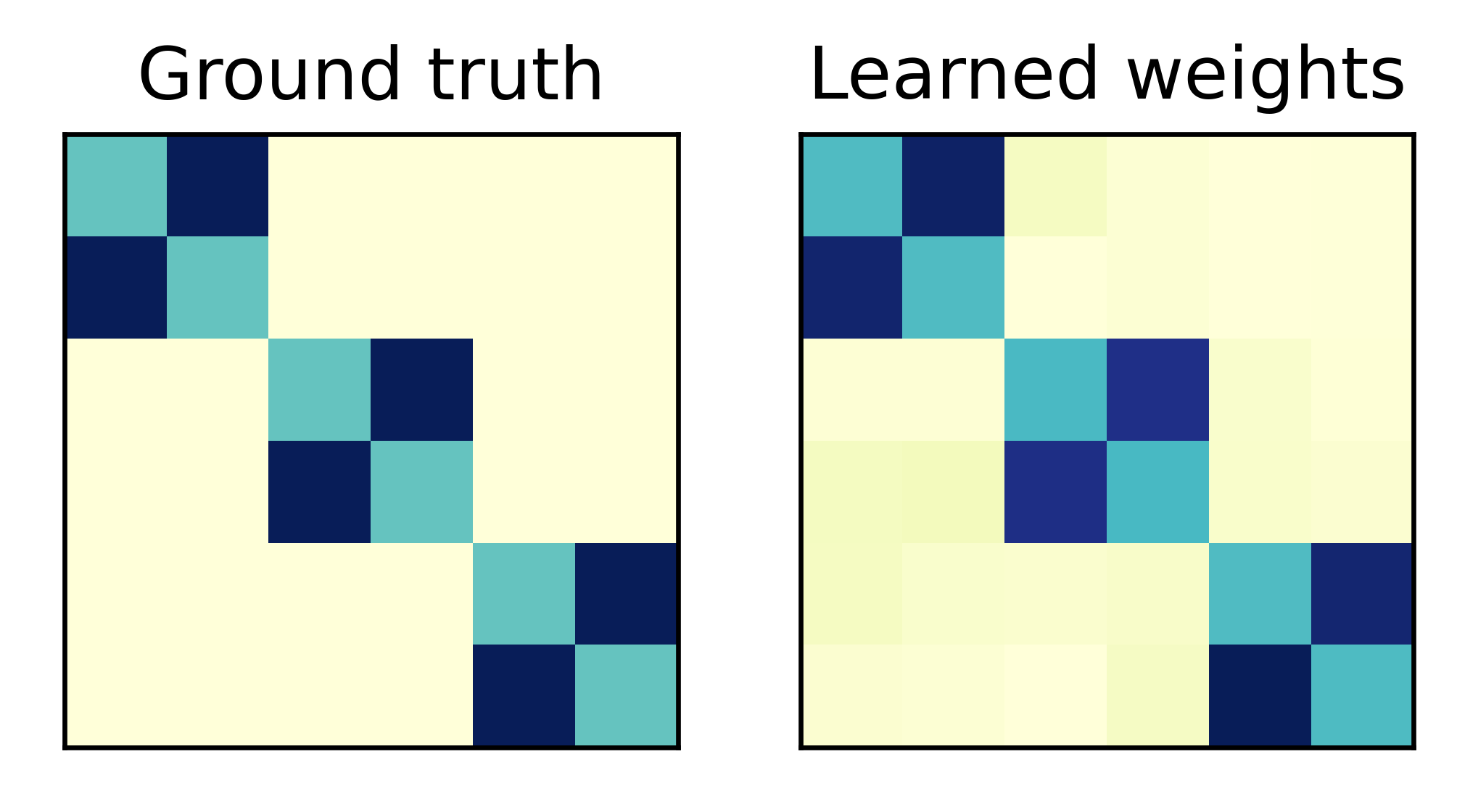}
  \caption{common noise}
\end{subfigure}
\begin{subfigure}{0.24\textwidth}
  \centering
  \includegraphics[width=4cm]{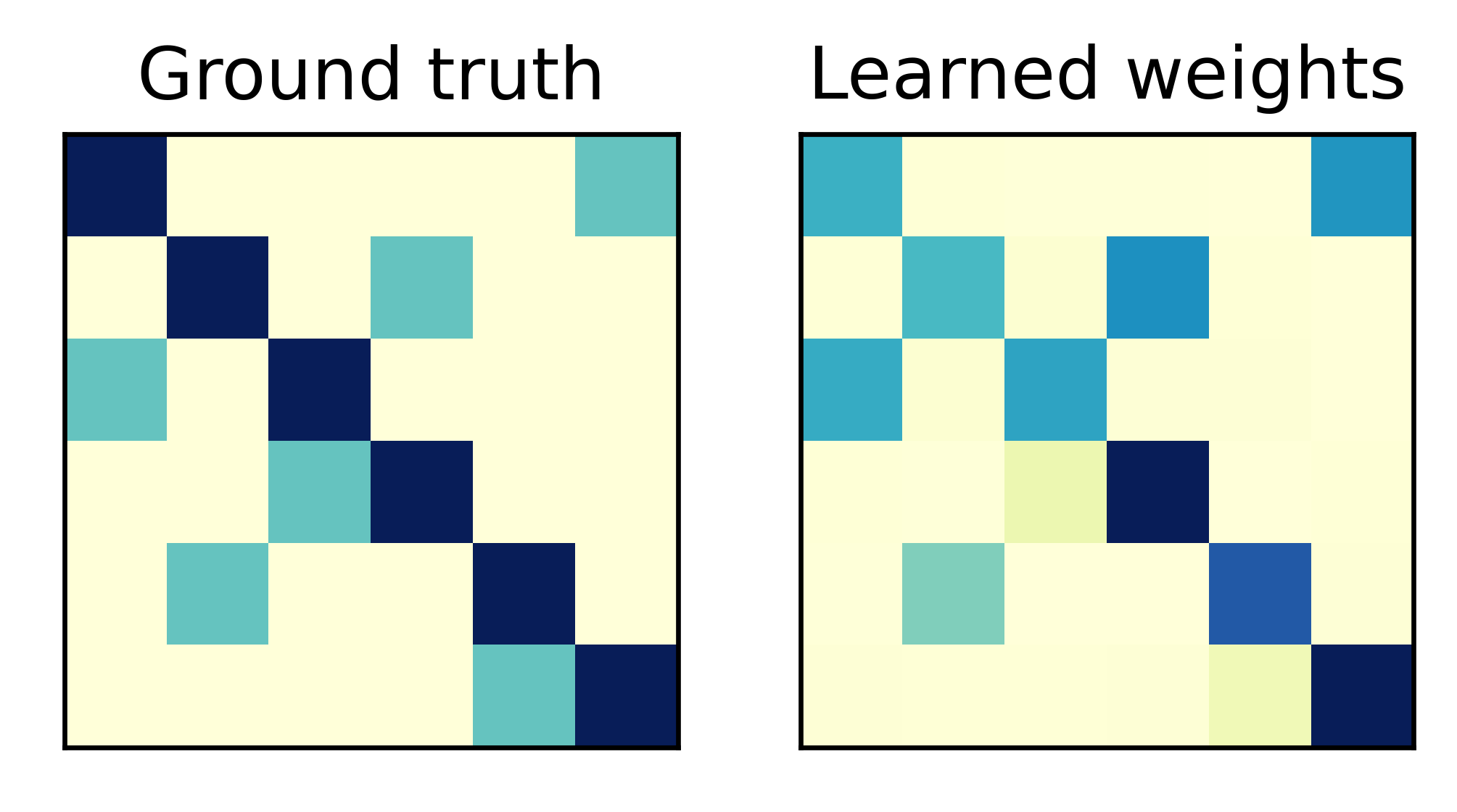}
  \caption{annotator 1}
\end{subfigure}
\begin{subfigure}{0.24\textwidth}
  \centering
  \includegraphics[width=4cm]{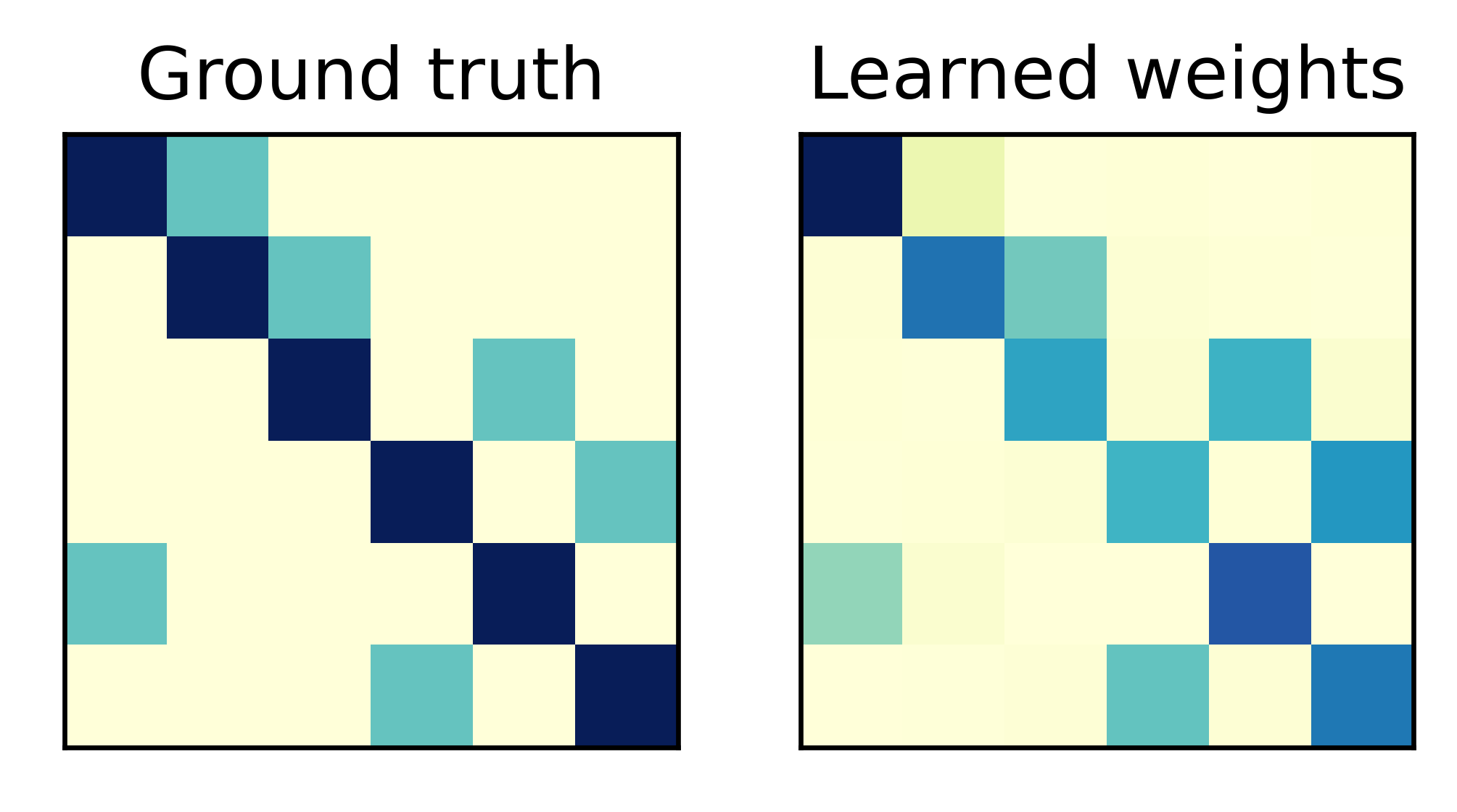}
  \caption{annotator 2}
\end{subfigure}
\begin{subfigure}{0.24\textwidth}
  \centering
  \includegraphics[width=4cm]{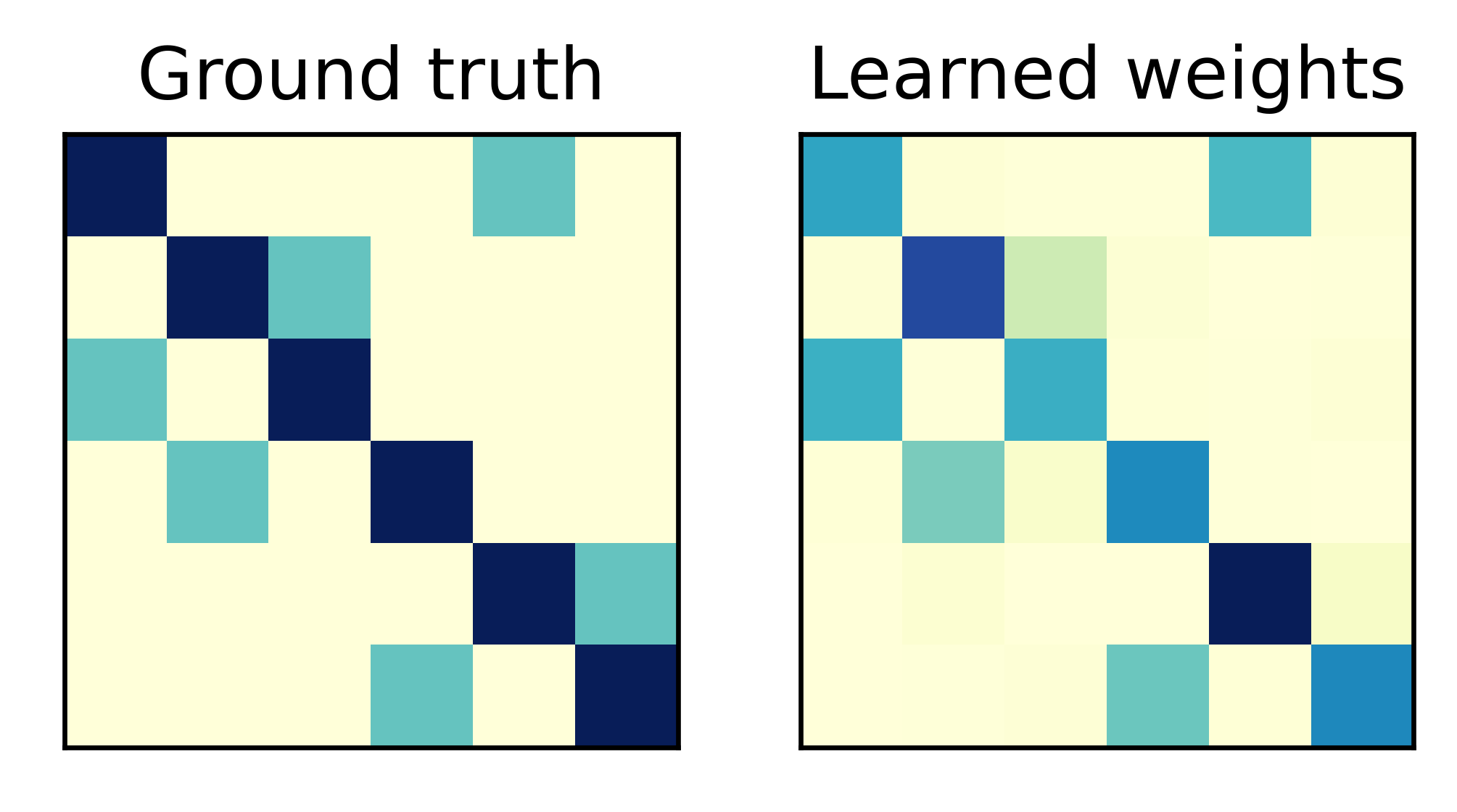}
  \caption{annotator 3}
\end{subfigure}
\caption{Comparison between ground truth confusion matrices and learned ones on Synthetic dataset. The top row is the result of asymmetric common noise. The bottom row is the result of symmetric common noise.}
\label{fig:syn_mat_recovery}
\vspace{-1.2em}
\end{figure*}

\section{C. Additional experiment results} 
\label{app:add_results}

\noindent\textbf{Results on Synthetic dataset.} Figure \ref{fig:syn_results} presents the test accuracy on Synthetic dataset under the same settings as we described in Section 3. We report mean and standard deviation of test accuracy on five runs.  The results align with our analysis in Section 3. Under the asymmetric confusion, our proposed approach is robust to the settings of common noise strength and the proportion of common noise. Under the symmetric confusion, all methods' performance is influenced (becomes worse); however, CoNAL still outperforms baselines with a large margin by differentiating the source of noise. 

Figure \ref{fig:syn_mat_recovery} shows the learnt weights of noise adaptation layers against the ground-truth confusion matrices on the Synthetic dataset. We set the common noise strength to 0.7 and the average proportion of common noise around 0.5. We reconstruct the confusion matrix from the learnt weights by normalizing them using softmax on each row. From the results, we can clearly observe most confusion matrices (especially the confusion matrix for common noise) are well recovered under both confusion settings.
\label{app:synthetic_result}

\noindent\textbf{Visualization of learnt confusion matrices on real-world datasets.}
\label{app:vis_real}
We provide visualization of learnt confusion matrices on both real-world datasets in Figure \ref{fig:conf_labelme} and \ref{fig:conf_music}. 
We can clearly observe these two types of learnt confusion matrices, i.e., for common confusion and individual confusion, capture different mistake patterns across annotators. For example, on LabelMe dataset, the commonly made mistake from \emph{inside city} to \emph{street} was covered by the learnt common confusion. The same observation is also obtained on the Music dataset, such as the common mistake from \emph{jazz} to \emph{blues} is reflected in our learnt common confusion matrix. On the other hand, the individual noise on annotators captures their own specific mistakes. For example, on LabelMe dataset, both annotator 1 and 2 confused about \emph{tall building} and \emph{inside city}; but this mistake does not appear in annotator 3, 4 and 5, nor the global confusion matrix. 

We also notice some low-quality annotators in the Music dataset, such as annotator 1 (with ground-truth annotation accuracy of 0.182) and annotator 5 (with ground-truth annotation accuracy of 0.108), whose annotations are almost random. By separately modeling the annotation noise at a per-annotation basis, our solution reduces the influence from such low-quality annotators in learning the common noise model and maintains the quality of inferred true labels overall.

\begin{figure*}[!htp]
\centering
\begin{subfigure}{1\textwidth}
  \centering
  \includegraphics[width=12.8cm]{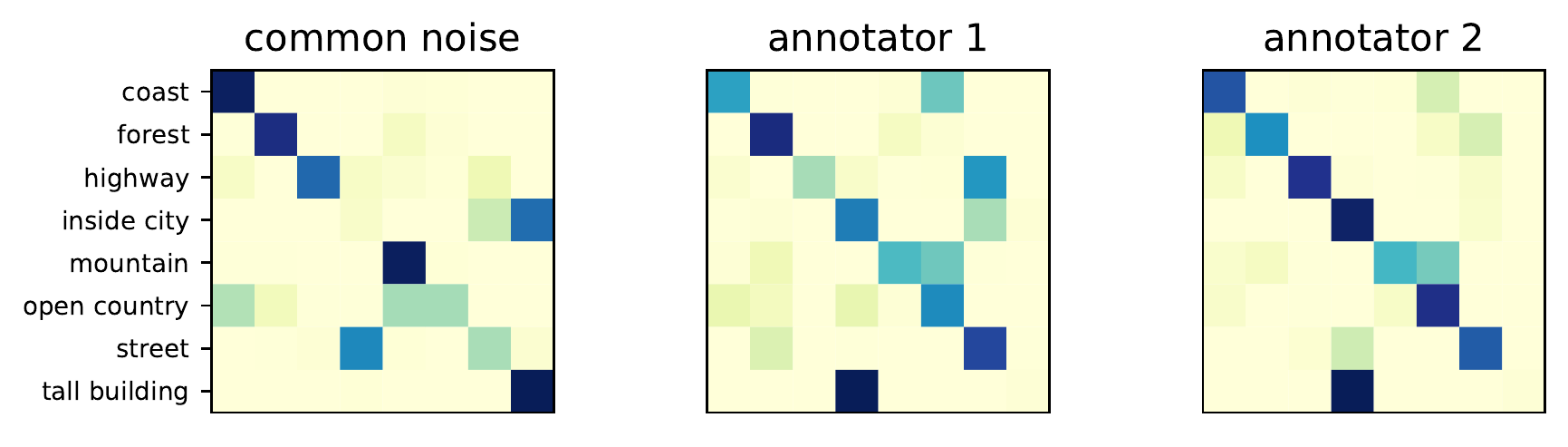}
\end{subfigure}

\begin{subfigure}{1\textwidth}
  \centering
  \includegraphics[width=12.8cm]{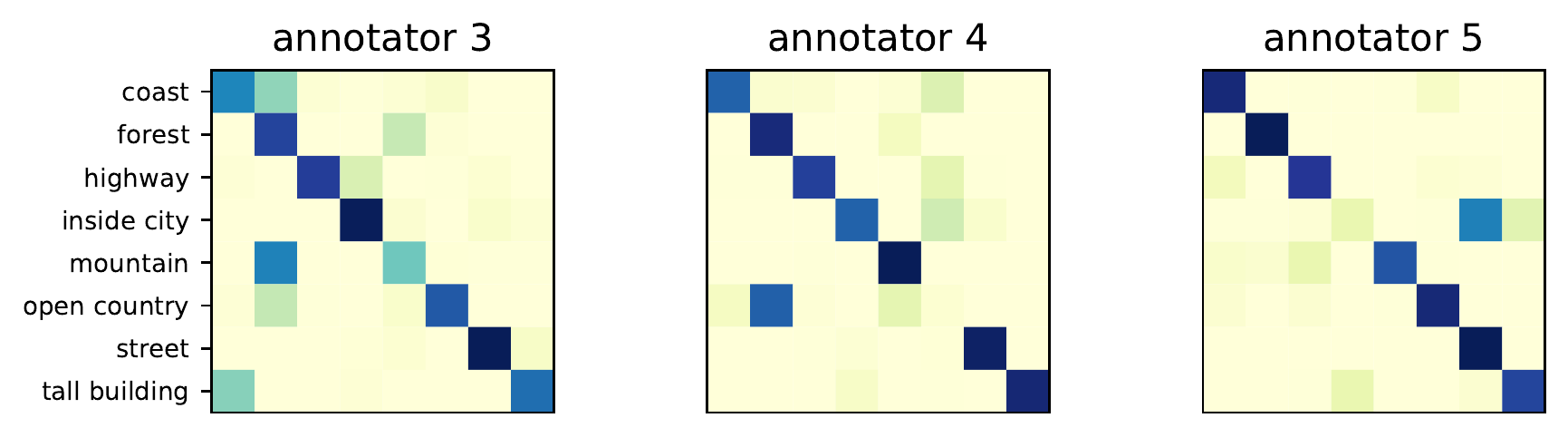}
\end{subfigure}
\caption{Learnt global confusion matrix and individual confusion matrices of 5 annotators on LabelMe dataset.}
\label{fig:conf_labelme}
\vspace{-1.2em}
\end{figure*}

We visualized the distribution of inferred proportion of common noise (i.e., $\omega^r_i$) across annotations to better understand how CoNAL differentiates the source of noise in individual annotations. We rank the instances by their average $\omega^r_i$ over all annotators who have labeled this instance in a descending order. Then we count the frequency of ground-truth labels in the top 50\% and bottom 50\% instances respectively and report the results in Figure \ref{fig:w_dist}. The larger the average $\omega^r_i$ in an instance is, the more likely the annotators made similar mistakes on it (i.e., the common confusion matrix can better explain the observed annotations on this instance). On LabelMe dataset, we can observe that annotators tend to make similar annotations on \emph{forest} (with ground-truth annotation entropy 0.660) and \emph{mountain} (with ground-truth annotation entropy 0.192), but make their own mistakes on \emph{open country} (with ground-truth annotation entropy 1.287) and \emph{inside city} (with ground-truth annotation entropy 1.116). In other words, the annotations on \emph{forest} and \emph{mountain} are much more consistent than those on \emph{open country} and \emph{inside city}. This observation can also be explained by the learnt confusion matrices. From the learnt global confusion matrix, we can observe confusion patterns in \emph{open country} and \emph{inside city} are quite scattered, and different annotators (e.g., all those five visualized annotators) have distinct confusions. While for \emph{forest} and \emph{mountain}, the global confusion matrix correctly maps them to the correct annotation, and individual annotators might occasionally make their own mistakes, e.g., annotator 3. 
Similar observations are also obtained on the Music dataset, where annotators tend to make similar mistakes on \emph{hiphop} and \emph{reggae}, and make their own distinct mistakes on \emph{jazz} and \emph{rock}.

\noindent\textbf{Discussion about overparameterized models.} To prove that the improved performance of our solution comes from its unique modeling of crowdsourced data other than simply an increased number of parameters to fit, we compare our model with the overparameterized DL-CL \cite{rodrigues2018deep}, which has a similar structure as ours to capture individual confusions, but without the notion of modeling common confusion. \citet{rodrigues2018deep} discussed that simply adding more parameters can make the output of the learnt classifier lose its interpretability as a shared ground-truth estimate across annotators, so that they only used one softmax layer for each annotator upon the classifier's output layer. We add another softmax layer for each annotator, to introduce more parameters but avoid losing the interpretability of the bottleneck layer, we name it as DL-CL\_Over.

\begin{table}[!htp]
\centering
\begin{tabular}{c @{\hspace{0.4\tabcolsep}} c @{\hspace{0.4\tabcolsep}} c @{\hspace{0.4\tabcolsep}} c @{\hspace{0.4\tabcolsep}}}
\toprule
 Model & DL-CL & DL-CL\_Over & CoNAL   \\ \midrule 
 \#Params in NAL & $R\times C^2$  & $2 \times R\times C^2$  & $(R + 1)\times C^2$ \\ \midrule
 
LabelMe & 83.27{\footnotesize $\pm 0.52$} & 82.34{\footnotesize $\pm 0.34$}   & 87.12{\footnotesize $\pm 0.55$}   \\ \midrule
Music   & 81.46{\footnotesize $\pm 0.53$}  & 80.47{\footnotesize $\pm 0.27$}  & 84.06{\footnotesize $\pm 0.42$}   \\ \bottomrule
\end{tabular}
\caption{Comparison with the overparameterized model.}
\label{table:over}
\vspace{-3mm}
\end{table}

We present the results on real-world datasets, along with the number of parameters in the noise adaptation layers (NAL). Even though DL-CL\_Over has the most number of parameters to fit, its performance did not increase but decreased, which indicates that blindly adding more parameters will not help model crowdsourced data. Our model adds a global noise adaptation layer, which has fewer parameters than DL-CL\_Over. The results prove the advantage of our model comes from its unique design to annotation confusions, but not simply more parameters to fit.

\begin{figure*}[!thp]
\centering
\begin{subfigure}{1\textwidth}
  \centering
  \includegraphics[width=12.8cm]{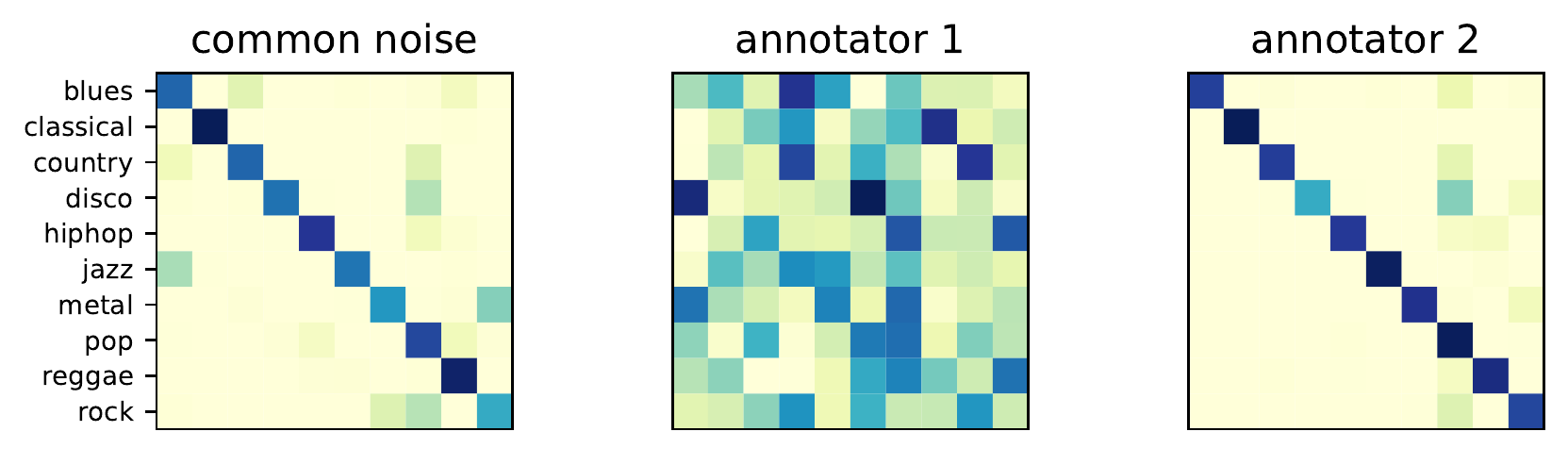}
\end{subfigure}

\begin{subfigure}{1\textwidth}
  \centering
  \includegraphics[width=12.8cm]{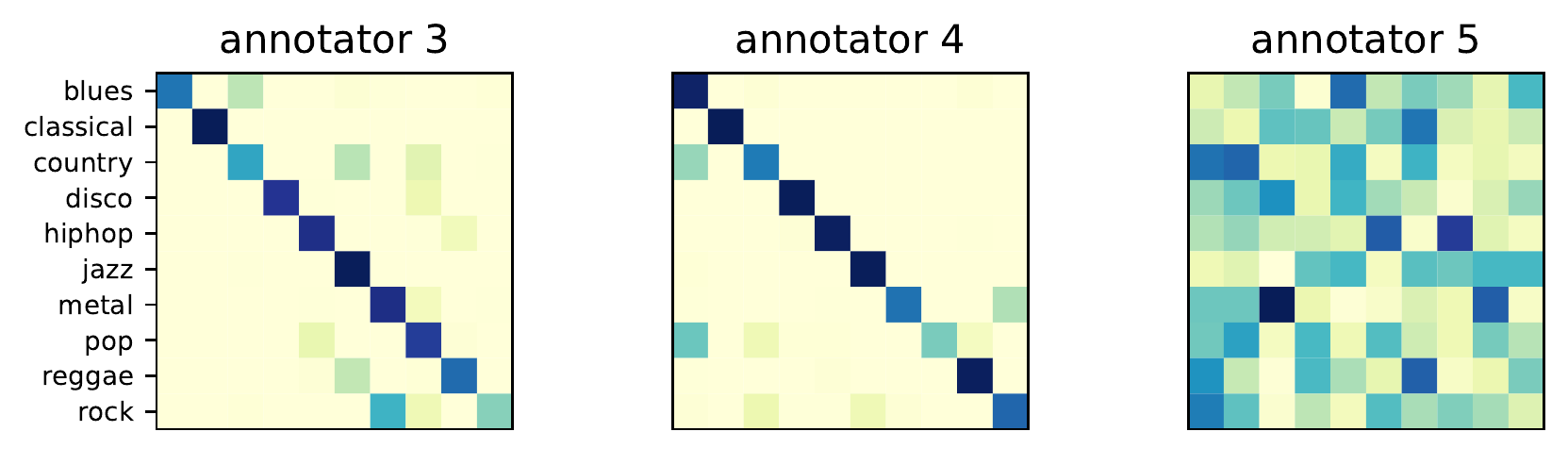}
\end{subfigure}
\caption{Learnt global confusion matrix and individual confusion matrices of 5 annotators on Music dataset.}
\label{fig:conf_music}
\vspace{-1.2em}
\end{figure*}

\begin{figure*}[!h]
\begin{subfigure}{1\textwidth}
  \centering
  \includegraphics[width=13.8cm]{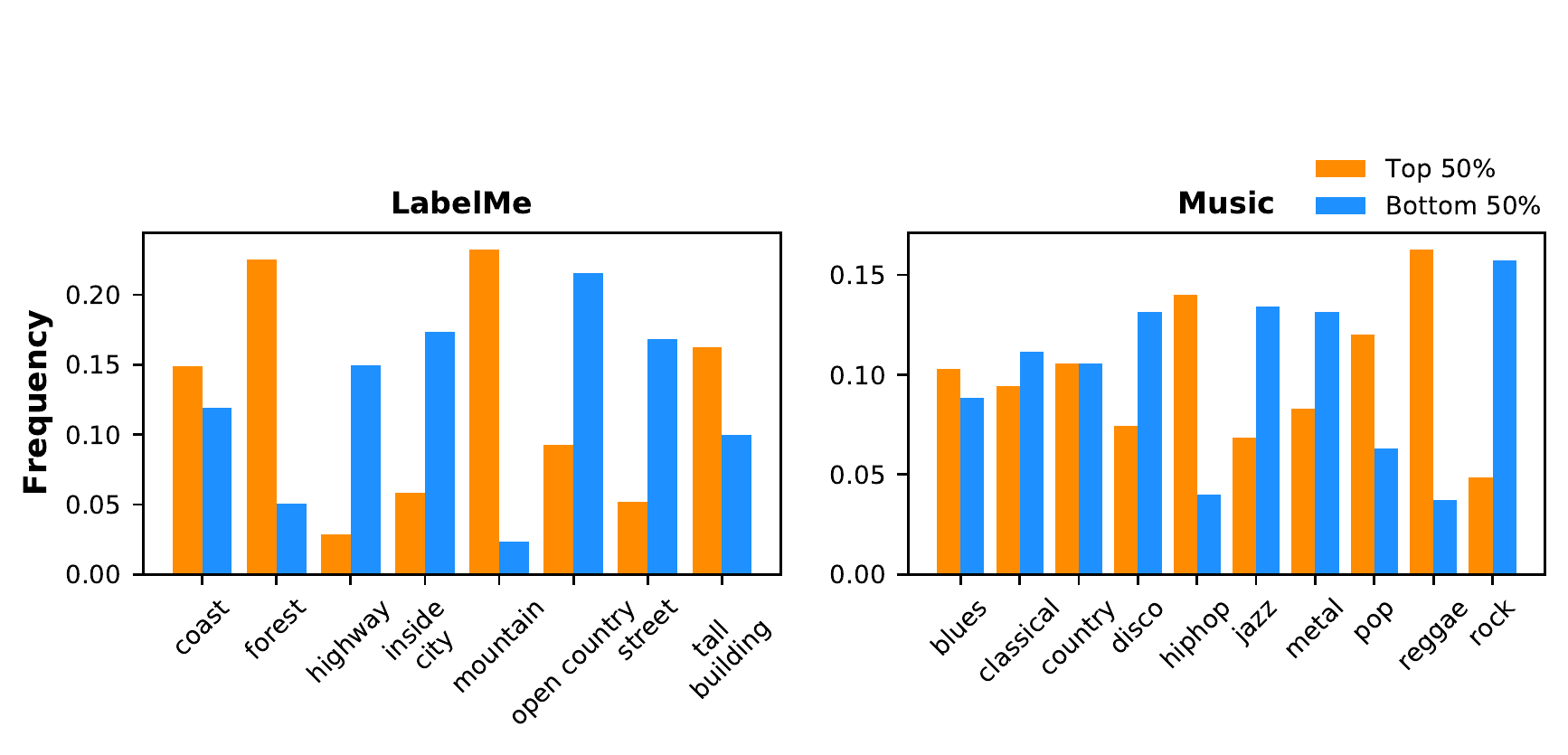}
\end{subfigure}
\caption{$\omega$-label distribution on real-world datasets. We rank the instances by average $\omega$ over all annotators and visualize the ground-truth label distribution of top 50\% and bottom 50\% instances.}
\label{fig:w_dist}
\vspace{-1.2em}
\end{figure*}
